\documentclass{article} 
\usepackage{iclr2022_conference,times}

\usepackage{hyperref}
\usepackage{url}
\usepackage[utf8]{inputenc} 
\usepackage{amsmath,amsfonts,bm}
\usepackage{booktabs}       
\usepackage{nicefrac}       
\usepackage{microtype}      
\usepackage{multirow}
\usepackage{multicol}
\usepackage{amssymb}
\usepackage{amsthm}
\usepackage{helvet}  
\usepackage{courier}  
\usepackage{enumitem}
\usepackage{graphicx}
\usepackage{subcaption}
\usepackage{bbding}
\usepackage{xspace}

\usepackage{algorithm}
\usepackage[noend]{algpseudocode} 
\usepackage{color}

\usepackage{booktabs}       
\usepackage{amsfonts}       
\usepackage{nicefrac}       
\usepackage{microtype}      
\usepackage{multirow}
\usepackage{multicol}  

\usepackage{subcaption}
\usepackage{bbding}
\usepackage{xspace}
\usepackage[american]{babel}

\newcommand{\LS}{\mathcal{L}}

\newcommand{\D}{\mathcal{D}}
\newcommand{\w}{\boldsymbol{w}}

\DeclareMathOperator{\E}{\mathbb{E}}
\DeclareMathOperator{\KL}{\boldsymbol{\mathbf{KL}}}
\newcommand{\R}[1]{\mathbb{R}^{#1}}
\newcommand{\norm}[1]{\left\|{#1}\right\|}

\newtheorem{theorem}{Theorem}
\newtheorem{lemma}[theorem]{Lemma}

\newtheorem{assumption}[theorem]{Assumption}

\usepackage{pifont}
\newcommand{\cmark}{\ding{51}}%
\newcommand{\xmark}{\ding{55}}%

\newcommand{\system}{\textsc{FedGKD}\xspace}
\newcommand{\systemled}{\textsc{FedGKD-Vote}\xspace}
\newcommand{\systemplus}{\textsc{FedGKD$^{+}$}\xspace}

\title{Local-Global Knowledge Distillation in Heterogeneous Federated Learning with Non-IID Data}

\author{Dezhong Yao, Wanning Pan \\
Huazhong University of Science and Technology \\
\texttt{\{dyao,pwn\}@hust.edu.cn} \\
\And
Yutong Dai\\
Lehigh University \\
\texttt{yud319@lehigh.edu}
\And
Yao Wan, Xiaofeng Ding, Hai Jin\\
Huazhong University of Science and Technology \\
\texttt{\{wanyao,xfding,hjin\}@hust.edu.cn}
\And
Zheng Xu \\
Google Research \\
\texttt{xuzheng@google.com} \\
\And
Lichao Sun \\
Lehigh University \\
\texttt{lis221@lehigh.edu}
}

\iclrfinalcopy 
\begin{document}

\maketitle

\begin{abstract}
Federated learning enables multiple clients to collaboratively learn a global model by periodically aggregating the clients' models without transferring the local data. However, due to the heterogeneity of the system and data, many approaches suffer from the {\em ``client-drift''} issue that could significantly slow down the convergence of the global model training. As clients perform local updates on heterogeneous data through heterogeneous systems, their local models drift apart. To tackle this issue, one intuitive idea is to guide the local model training by the global teachers, {\em i.e.,} past global models, where each client learns the global knowledge from past global models via adaptive knowledge distillation techniques. Coming from these insights, we propose a novel approach for heterogeneous federated learning, namely \system, which fuses the knowledge from historical global models for local training to alleviate the {\em ``client-drift''} issue. In this paper, we evaluate \system with extensive experiments on various CV/NLP datasets ({\em i.e.,} CIFAR-10/100, Tiny-ImageNet, AG News, SST5) and different heterogeneous settings. The proposed method is guaranteed to converge under common assumptions, and achieves superior empirical accuracy  in fewer communication runs than five state-of-the-art methods.

\end{abstract}

\section{Introduction}
Clients collaboratively learn a global model without transferring their local data in federated learning (FL). The data distribution on clients are often non-IID (independent and identically distributed) in practice, also known as data heterogeneity, which is a challenge in FL and has drawn much attention
in recent studies~\citep{karimireddy2019scaffold, wang2020tackling,wang2021field}. 
Moreover, heterogeneity can be caused by the empirical sampling size of clients data, the joint data-label distribution, communication efficiency  and capacities of computation, and many others~\citep{kairouz2019advances}.
Since the local data on each client are not sampled from the global joint distribution of all clients~\citep{li2019convergence}, the local objectives are different and may not share common minimizers.
Even every communication starts from the same global model, clients' local models will \textit{drift} towards the minima of their local objectives, and the aggregated global model may not be the optimum of the global objective.
Such {\em ``client-drift''} phenomenon not only degrades performance, but also increase the number of communication rounds~\citep{karimireddy2019scaffold}. 
Thus, it is desirable to explicitly handle the heterogeneity during training.

In light of the client-drift problem, many efforts have been made mainly from two aspects. (1) One focuses on using additional data information to address the model drift issue induced by non-IID data. As illustrated in Tab.~\ref{tbl:alg_comp}, FedDistill~\citep{seo2020federated} shares local output logits on training dataset, and FedGen~\citep{zhu2021data} shares local label count information. While the additional information sharing methods may face the risk of privacy leakage. FedDF~\citep{lin2020ensemble} requires additional proxy data in server for ensemble distillation. Though the efficacy of model aggregation is improved, the additional proxy dataset may not always be available. And the inherent heterogeneity among local models is not fully addressed by only refining the global model, which may affect the quality of the knowledge ensemble, especially when the data distribution shift exists~\citep{khoussainov2005ensembles}.  (2) Another line of work aims at local-side regularization. FedProx~\citep{li2018federated} adds a proximal term in the local objective, and SCAFFOLD~\citep{karimireddy2019scaffold} uses control variate to correct the client-drift. FedDyn~\citep{acar2020federated} proposes a dynamic regularization to ensure the local optima is consistent. By injecting a projection head to the model, MOON~\citep{li2021model} uses the model-level contrastive learning method to correct the local training of individual clients.
In this paper, we try to use knowledge distillation technique to solve the client-drift problem, in which does not need to share additional information, use additional proxy data, or modify the model structure.

\begin{table}[!t]
\centering
\caption{Comparison of federated learning approaches.}
\renewcommand\arraystretch{1.1}

\begin{tabular}{lccc}
\toprule
\textbf{Method}  & \textbf{{Add. Inf. Sharing}} & \textbf{{Proxy Data}} & \textbf{{Model Modification}}  \\
\hline
  FedAvg~\citep{mcmahan2017communication}  & \xmark & \xmark & \xmark \\
  FedProx~\citep{li2018federated}  & \xmark & \xmark & \xmark\\
  FedDistill~\citep{seo2020federated} & \cmark & \xmark & \xmark\\
  FedGen~\citep{zhu2021data} & \cmark & \xmark & \xmark \\
  FedDF~\citep{lin2020ensemble} & \xmark & \cmark & \xmark \\
  MOON~\citep{li2021model}  & \xmark & \xmark & \cmark\\
  \system (Ours) & \xmark & \xmark & \xmark \\
  \bottomrule
  \end{tabular}
 \label{tbl:alg_comp}
\end{table}

Motivated by the aforementioned problems, we propose a novel ensemble-based global knowledge distillation method, named \system, which fuses the knowledge from past global models to tackle the client drift in training. Specifically, \system combines multiple  classifiers (partial knowledge) to form a compact representation of global knowledge and transfer it to clients, guiding local model training. 
A knowledge distillation loss based on compact global knowledge representation is imposed on each client to mitigate the client-drift issue.
Overall, the primary contributions of this paper are as follows.
\begin{itemize}
    \item We introduce an ensemble-based knowledge distillation technique to transfer the information of historical global models to local model training for preventing over-biased local model training on non-IID data distribution.
    \item We provide a generalized and simple method for federated knowledge distillation, which does not require additional information sharing, proxy data, or model modification. The communication cost and training overhead are at the same complexity level as that of FedAvg. Moreover, the proposed method is compatible with many privacy protection methods like differential privacy in federated learning.
    \item We use extensive experiments and analysis on various datasets (i.e. CIFAR-10/100, Tiny-ImageNet, Ag News and SST-5) and settings (i.e. different level of non-IID data distribution and participation ratio) to validate the effectiveness of the proposed \system, and compare \system with several state-of-the-art methods.
\end{itemize}

\section{Background and Preliminaries}


\paragraph{Federated Learning}(FL) is a communication-efficient distributed learning framework where a subset of clients performs local training before aggregation~\citep{KonecnyMRR16, mcmahan2017communication}, which can train a machine learning model without sharing clients' private data~\citep{bonawitz2017practical}, therefore the data privacy will be preserved. In FL, one critical problem is on handling the unbalanced, non-independent and identically distributed (non-IID) data, which are very common in the real world~\citep{zhao2018federated,sattler2019robust,li2019convergence}.
At the application level, FL has been applied to a wide range of real applications, such as  biomedical~\citep{brisimi2018federated}, healthcare~\citep{xu2021federated}, finance~\citep{yang2019ffd}, and smart manufacturing~\citep{hao2019efficient}.

Suppose there are $N$ clients, denoted as $C_1,\ldots,C_N$. Client $C_k$ stores a local dataset $D_k$ with a distribution $\D_k$. In general, FL aims to learn a global model weight $\w$ over the dataset $\mathcal{D} = \cup \{D_k\}_{k=1}^{n_K}$, where $n_K$ is the number of participants in each training round. The objective of FL is to solve the optimization~\citep{mcmahan2017communication}:
\begin{equation}
  \min_{\w \in \R{d}} f(\w) := \sum^{K}_{k=1} p_k F_k(\w)\,, \label{eq:fl}
\end{equation}
where $p_k=n_k/\sum_{k=1}^Kn_k$, $n_k=|D_k|$ and $F_k(\w)=\frac{1}{n_k}\sum_{i=1}^{n_k}\mathcal{L}(h_k(\w;x_{ki}),y_{ki})$ is the local objective function at the $k$-th client. Here, $h_k:\R{d}\times \R{d'}\to \R{\mathcal{C}}$ represents the local model and $\mathcal{L}: \R{\mathcal{C}}\times \R{\mathcal{C}} \to \R{}_{+}$ is the loss function. Moreover, $(x_{ki}, y_{ki})$ corresponds to a sample point from the $k$-th client for all $i\in[n_k]$ and $k\in[K]$.

\paragraph{Knowledge Distillation}(KD) is referred as teacher-student paradigm that a cumbersome but powerful teacher model transfers its knowledge through distillation to a lightweight student model~\citep{buci2006model}. KD aims to minimize the discrepancy between the logits outputs from the teacher model $\w^T$ and the student model $\w^S$ with a dataset $D$~\citep{hinton2015distilling}.
The discrepancy could be measured by Kullback-Leibler divergence:
\begin{equation}
    \min_{\w^S} \E_{x \sim \mathcal{D}}[h(\w^T;x)||h(\w^S;x)]. \label{eq:kl}
\end{equation}
KD has shown great success in various FL tasks. First, many studies focus on data heterogeneity (non-IID) problem. 
FedDistill~\citep{seo2020federated} performs KD to refine the local training by sharing local logits. 
FedGen~\citep{zhu2021data} learns a global generator to aggregate the local information and distills global knowledge to users. 
FedDF~\citep{lin2020ensemble} uses the averaged logits of local models on proxy data for aggregation. This approach treats the local models as teachers and transfers their knowledge into a student (global) model to improve its generalization performance. 
Second, some works aim to reduce the communications cost. CFD~\citep{sattler2020communication} sends the soft-label predictions on a shared public data set to server to save communication cost. 
Besides these, FedLSD~\citep{lee2021preservation} and FedWeIT~\citep{yoon2021federated} try to solve continual learning problem in FL. In this paper, we focus on non-IID data scenarios.

In general, previous works require either additional local information~\citep{seo2020federated,zhu2021data} or proxy data~\citep{lin2020ensemble,sattler2021fedaux,sattler2020communication}. However, as discussed in~\citep{mcmahan2017communication,bonawitz2017practical,li2018federated}, the partial global sharing of local information violates privacy, and using globally-shared proxy data should be cautiously generated or collected. Comparing with previous methods, our approach does not require any local information sharing or proxy data.

\section{Local-Global Knowledge Distillation in Heterogeneous Federated Learning}

\subsection{Motivation}
\system is motivated by an intuitive idea: the global model is able to extract a better feature representation than the model trained on a skewed subset of clients.
More precisely, under non-IID scenarios, local datasets across clients are nonidentical distributed, so local data fail to represent the overall global distribution. For example, given a model trained on airplane and bike images, we cannot expect the features learned by the model to recognize birds and dogs. Therefore, for non-IID data, we should control the drift and bridge the gap between the representations learned by the local model and the global model. KD has been verified to be a powerful knowledge transferring method~\citep{hinton2015distilling}. Additionally, the ensemble~\citep{PolyakEnsemble} of multiple historical global models can further enhance the power of global knowledge to handle the local model drift problem by providing a more comprehensive view of the global data distribution.

\subsection{\system}
Our core target is to reduce local model drift through global model distillation. An overview of this procedure is illustrated in Fig.~\ref{fig:framework}. We first introduce a base method that uses the latest global model to guide local training. Then we extend the method to use multiple historical global models for guidance.

\begin{figure}[t]
\centering
    \includegraphics[width=0.6\linewidth]{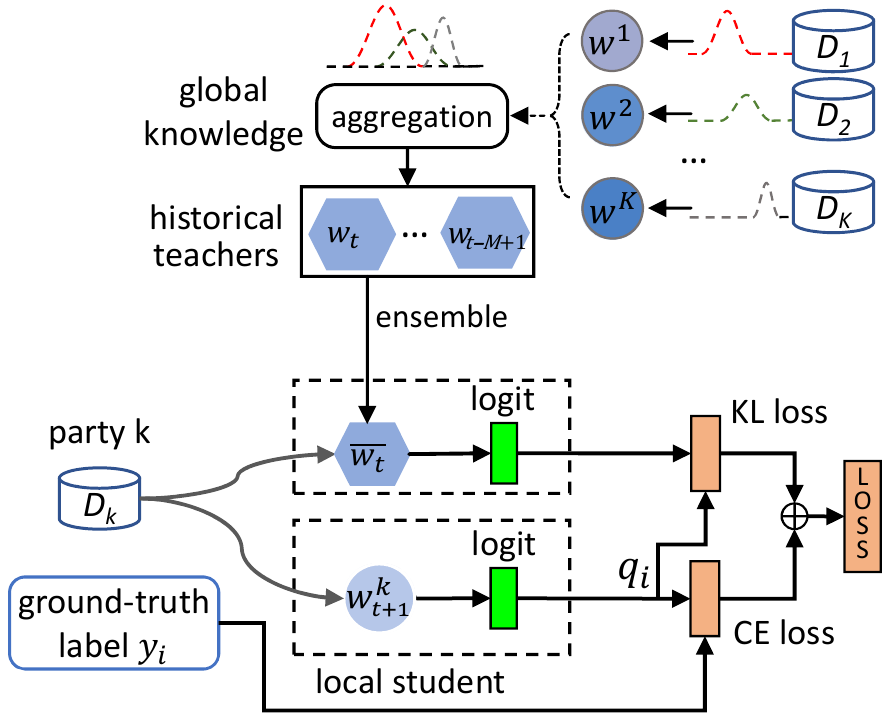}
    \caption{An overview of \system: an ensemble model $\overline{\w_t}$ is learned by aggregating information from historical global models. The ensemble model is then sent to sampled client, whose knowledge is distilled to local models for a good feature distribution.}\label{fig:framework}
\end{figure}

\paragraph{Global Knowledge Distillation}
\system\ is proposed for transferring the global knowledge to local model training, where each client's data has different data distribution.
Due to the nature of non-IID data distribution,
the local model trained may be biased towards its own dataset, i.e, over-fitting its local dataset, which causes the complexity and difficulty of global aggregation.
In order to mitigate the data and information gaps among local models, a novel idea stems from the knowledge distillation technique that regularizes local training with global model output.
Suppose the $\w_t$ is the parameter of the global model at round $t$, we formulate the local objective of \system\ for the $k$th client at $t+1$ round as
\begin{align}
    &\min_{\w} { \underbrace{ F_k(\w) \vphantom{\sum_{i=1}^{n_k}} }_\text{CE loss} + 
    \underbrace{ \frac{\gamma}{2n_k}\sum_{i=1}^{n_k} \KL(h_k(\w_t,x_{ki}) || h_k(\w,x_{ki}))}_\text{KD loss} }, \label{eq:obj_sin} 
\end{align}
where $\gamma > 0$.




\paragraph{Historical Global Knowledge Distillation}

Motivated by~\citep{xu2020improving}, we further utilize the averaged parameters of historical global models for ensemble.
We define $\overline{\w_t} = \frac{1}{M}\sum_{m=1}^{M} \w_{t-m+1}$ as the averaged weight of the global, where $M$ is the buffer size. In round $t+1$, $M$ latest global models are averaged to a fused model $\overline{\w_t}$ on server and send to (sampled) clients. As shown in Fig.~\ref{fig:framework}, the local model $\w^k_{t+1}$ is trained under the guidance of the latest $M$ round global models to avoid local bias training. 
The modified local objective function for the $k$th client is defined as 
\begin{align}
  &\min_{\w} {F_k(\w) + \frac{\gamma}{2n_k}\sum_{i=1}^{n_k} \KL(h_k(\overline{\w_t},x_{ki}) || h_k(\w,x_{ki}))}.\label{eq:obj_avg}
\end{align}
We also develop the \systemled for regularizing the local models with past $M$ models output logits and the local objective is given as
\begin{equation}
  \min_{\w} {F_k(\w)\!+\!\sum_{m=1}^{M}\!\frac{\gamma_m}{2n_k}\!\sum_{i=1}^{n_k}\KL(h_k(\w_{t\!-\!m\!+\!1}\!,x_{ki}) || h_k(\w\!,x_{ki}))} , \label{eq:obj_his}
\end{equation}
where $\gamma_1,\cdots, \gamma_M$ are the distillation coefficients for different models.
Note that in numerical experiments, we found that using the global model from the last round  can already achieve a good performance, which is also easier to set up the hyper-parameter. 

In Algorithm~\ref{alg:fedgkd}, we give a detailed description for \system and \systemled. Both of \system and \systemled need to buffer the global models on the server, and \system average the global models weights directly so that only  $\overline{\w_t}$ and $\w_t$ should be sent to the clients. The communication cost of \system is twice compared to FedAvg if $M>1$ and the same as FedAvg if $M=1$. For \systemled, the communication cost would be $M$ times. Consider the communication rounds $T\gg M$ in realistic, the communication cost of \systemled is still $\mathcal{O}(T)$. Specially, for NLP fine-tuning tasks, we'd suggest set M as 1.

\begin{algorithm}[tb]
\begin{algorithmic}[1]
	\State {\bf Notations.} total number of clients $K$, server $S$, total communication rounds $T$, local epochs $E$, fraction of participating clients $C$, learning rate $\eta$, buffer size $M$, $\mathcal{B}$ is a set holding client’s data sliced into batches of size $B$. 
 	\State {\it //On the server side}
 	\Procedure{ServerExecution}{} \Comment{Server Model Aggregation}
	    \State Initial global weight $\w_0$
	    
	    \For{each communication round $t= 1,\ldots,T$}
	        \State $K^t \leftarrow$ random sample a set of $C \cdot K$  clients
    	    \If{\systemled}
    	    \State Send the global weight $\w_t,\ldots,\w_{t-M+1}$ to the selected clients
    	    \Else
    	    \State Send the global weight $\w_t$ and ensembled global weight $\overline{\w_t}$ to the selected clients
    	    \EndIf
    	    \State Buffer the global model with M size
	    
	    \EndFor
		\For{each client $k \in K_t$ {\bf in parallel}}
		\State $\w_{t+1}^k \leftarrow 
		\textsc{ClientUpdate}(k, \w_t)$
		\State $\w_{t+1} \leftarrow \sum_{k=1}^K \frac{n_k}{n}\w_{t+1}^k$
		\EndFor
 	\EndProcedure
	
	\State {\it //On the data owners side}
	\Procedure{ClientUpdate}{$k, \w_t$}\Comment{Local Model Training}
	\State $\w \leftarrow \w_t$
	\For{each local epoch $i = 1,\ldots,E$}
	    \For{each batch $b\in \mathcal{B}$}
		\State $\w \leftarrow \w - \eta \cdot \triangledown \LS(\w, b)$ \Comment{Update Local Weights via Eq.\ref{eq:obj_avg} or Eq.\ref{eq:obj_his}} 
		\EndFor
	\EndFor
	\State {\bf Return} $\w$ back to server
	\EndProcedure
\end{algorithmic}
\caption{\system: Global Knowledge Distillation in Federated Learning}
\label{alg:fedgkd}
\end{algorithm}

\section{\system Analysis}
In this section, we provide multiple perspectives to understand our proposed approach.

\subsection{Convergence Analysis}

Before formally state the convergence result, we introduce following notations. For any positive integer $a>0$, denote $[a]=\{1, 2,\ldots, a\}$. For any vector $v$, $[v]_j$ represents its $j$th coordinate. To perform convergence analysis, assumptions made on the properties of the client functions and the Algorithm are summarized in Assumption~\ref{ass.fun} and Assumption~\ref{ass.algo} respectively.


 \begin{assumption}\label{ass.fun}
 \item
 \begin{enumerate}
     \item (Lower-bounded eigenvalue) For all $k\in[K]$, $F_k(\w)$ is L-smooth, i.e., $\norm{\nabla F_k(\w) - \nabla F_k(\w)}\leq L\norm{\w-\w'}$ for some $L>0$ and all $(\w,\w)\in \R{d}\times\R{d}$. Furthermore, we assume the minimal eigenvalue of the Hessian of the client loss function $\nabla^2 f_k(w)$ is uniformly bounded below by a constant $\lambda_{\min}\in\mathbb{R}$.
     \item (Bounded dissimilarity) For all $k\in[K]$ and any $\w\in\R{d}$, $\E_k[\norm{\nabla F_k(\w)}^2]\leq B^2\norm{\nabla f(\w)}^2,$ where the expectation is taken with respect to the client index $k$.
    \item For all $k\in[K]$, $h_k$ outputs a probabilistic vector, i.e, $[h_k(.;.)]_j\geq 0$ and $\sum_{j=1}^{C}[h_k(.;.)]_j=1$ for all $j\in[C]$. Furthermore, we assume for any data point $x\in\R{d'}$, $\w\in\R{d}$, and $k\in[K]$, there exits a constant $\delta>0$ such that $\min_{j\in\{1,\cdots, C\}} [h_k(\w,x)]_j \geq \delta >0$. And there exits a constant $L_h>0$, $|h_k(\w,x) - h_k(\w',x)|\leq L_h \norm{\w-\w'}$ for any  $(\w,\w)\in \R{d}\times\R{d}$. 
 \end{enumerate}
 \end{assumption}
 We remark that Assumption~\ref{ass.fun}.1 and Assumption~\ref{ass.fun}.2 are also made in \citep{li2018federated}. Assumption~\ref{ass.fun}.3 is not stringent. For example, consider $h_k$ being a neural network for a classification task, and let the last layer being softmax, which is 1-Lipschtiz. As long as other layers (see \citep{NEURIPS2018_d54e99a6} for examples) are Lipschtiz, then by the fact that finite composition of Lipschitz continuous function gives the Lipschitz function, $h_k$ is also Lipschitz. Indeed $F_k$ is already assumed to be L-smooth, which implies $h_k$ is locally Lipschtiz and we only require $h_k$ to be Lipshctiz along the iterate sequence generated by the algorithm. 
 
\begin{assumption}\label{ass.algo}
  At the $t$-th round, the $k$-th client (if selected) solves the optimization problem  $\w^{k}_{t+1}\approx\arg\min_{\w\in\R{d}}\;m(\w;\w_{t})$ approximately with 
  \begin{equation}
    m(\w;\w_{t}):=F_k(\w) + \frac{\gamma}{2n_k}\sum_{i=1}^{n_k} \KL(h_k(\w_t,x_{ki}) || h_k(\w,x_{ki})), \nonumber
  \end{equation}
  such that 
\begin{equation}\label{eq:inexactness}
    \norm{\nabla F_k(\w^{k}_{t+1}) + \frac{\gamma L_h}{\delta}(\w_{t+1}^k-\w^t)}\leq \eta \norm{\nabla F_k(\w_{t})}\,,
\end{equation}
where $\eta\in [0,1)$ and $\gamma>0$.
\end{assumption}
The Lemma~\ref{lemma:wellposed} in the supplementary justifies the well-posedness of the Assumption~\ref{ass.algo}. The following theorem gives the rate of convergence to a stationary point.

\begin{theorem}\label{thm:convergence}
Let Assumption~\ref{ass.fun} and Assumption~\ref{ass.algo} hold.  Assume for each round, a subset of $S_t$ clients are selected, with $|S_t|=S$ and the $k$th client is selected with the probability $p_k$. If $\gamma, \eta$, and $S$ are chosen to satisfy:
\begin{enumerate}
    \item $\kappa := \frac{\gamma L_h}{\delta} + \lambda_{\min}>0$;
    \item $\rho>0$, where
    \begin{align*}
         \rho & = \frac{\delta}{\gamma L_h}\left(1-\eta B-\frac{LB(1+\eta)}{\kappa}\right) \\
         & - \frac{1}{\kappa}\left(\frac{\sqrt{2}B(1+\eta)}{\sqrt{S}} + \frac{L(1+\eta)^2B^2}{2\kappa} \right.
         + \left. \frac{(2\sqrt{2S}+2)LB^2(1+\eta)^2}{\kappa S}\right),
    \end{align*}

\end{enumerate}
then after $T$ rounds,
$$
\min_{t\in[T]}\E[\norm{\nabla f(w^t)}] \leq \frac{f(w^0)-f(w^*)}{\rho T}.
$$
\end{theorem}

\subsection{Global Knowledge Distillation for Biased Local Model Regularization}
As shown in Fig.~\ref{fig:cifar10-tsne}, we compare the local model and the global model of FedAvg, FedProx, and \system at the last round trained on CIFAR-10 by visualizing the penultimate layer representations of the test dataset. As the features are clustered by classes in T-SNE, a good model(classifier) can give more clear separation in the feature points. We notice that in Fig.~\ref{fig:cifar10-tsne_avg_client}, the local model can barely classify the entire classes due to the non-IID data distribution. And the global model is affected by the local model as shown in Fig.~\ref{fig:cifar10-tsne_avg_master}, since the boundary of different classes is vague due to the model drift. A clearer boundary can be observed in \system as shown in  Fig.~\ref{fig:cifar10-tsne_gkd_client}, where the improvement is brought by the knowledge distillation from global ensemble model, and we observe the improvement of local models leads to a better feature representation the global model of \system as shown in Fig.~\ref{fig:cifar10-tsne_gkd_master}.

\begin{figure*}[tb]
\centering
\subfloat[Global Model of FedAvg]{\label{fig:cifar10-tsne_avg_master}
\includegraphics[width=1.72in]{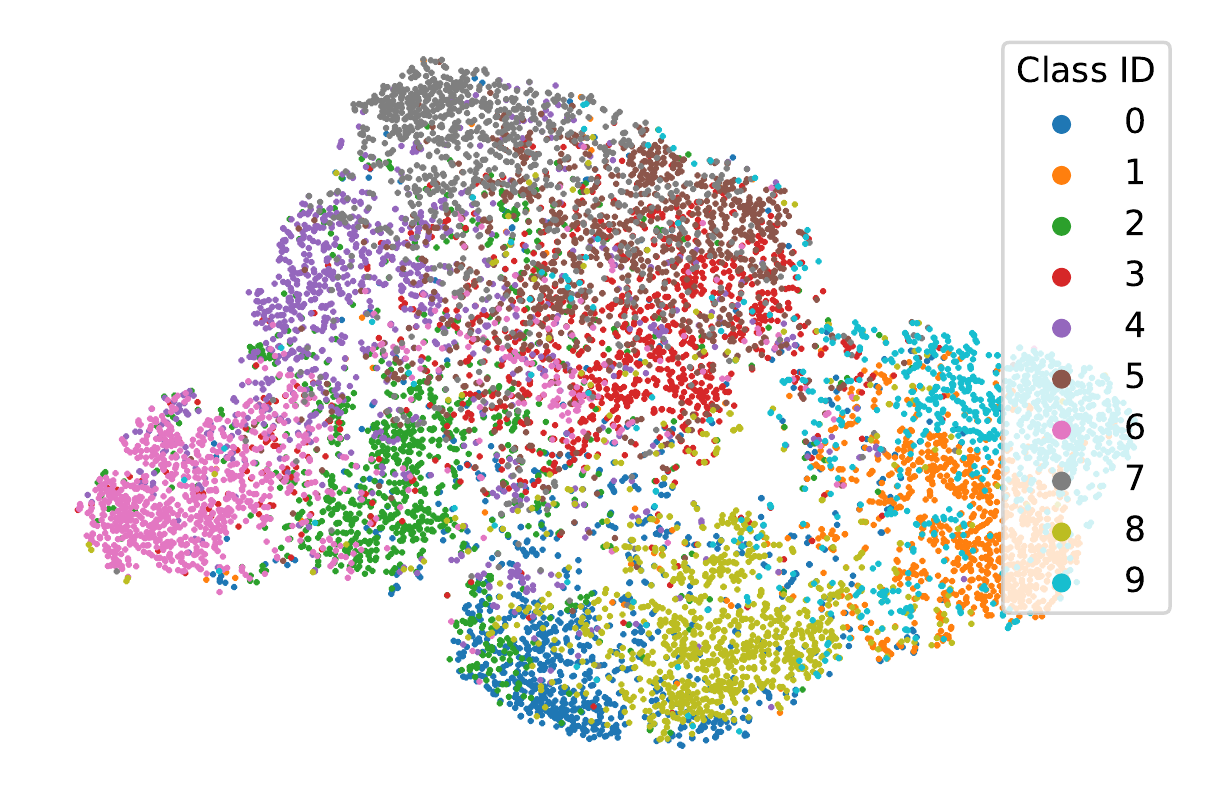}}\ \ \
\subfloat[Global Model of FedProx]{\label{fig:cifar10-tsne_prox_master}
\includegraphics[width=1.72in]{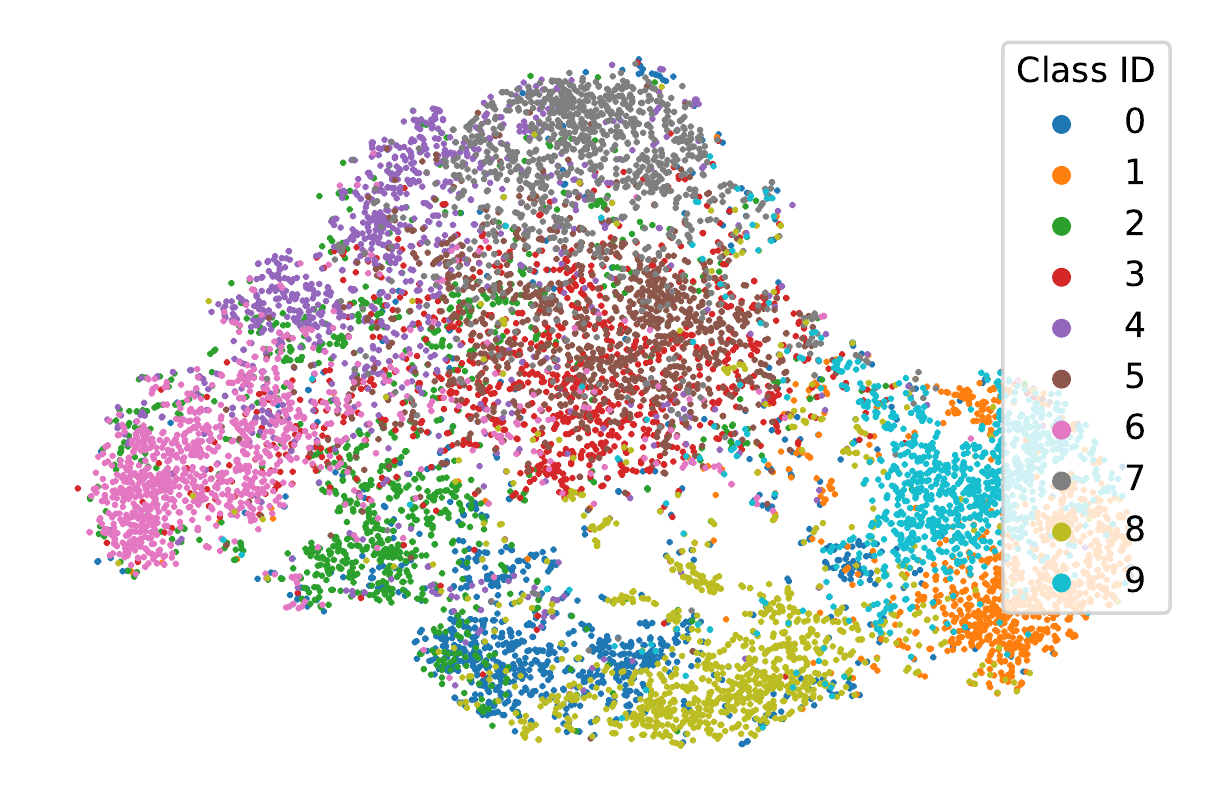}}\ \ \
\subfloat[Global Model of FedGKD]{\label{fig:cifar10-tsne_gkd_master}
\includegraphics[width=1.72in]{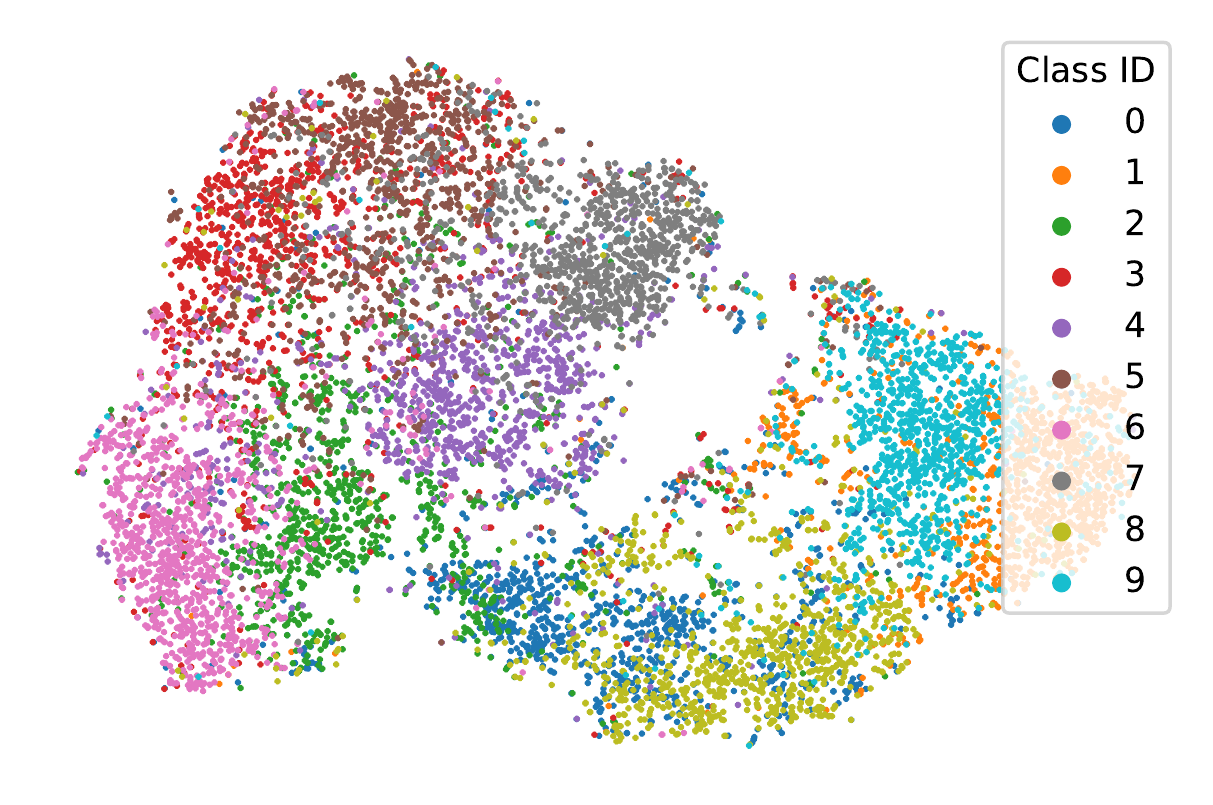}}\ \ \
\subfloat[Local Model of FedAvg]{\label{fig:cifar10-tsne_avg_client}
\includegraphics[width=1.72in]{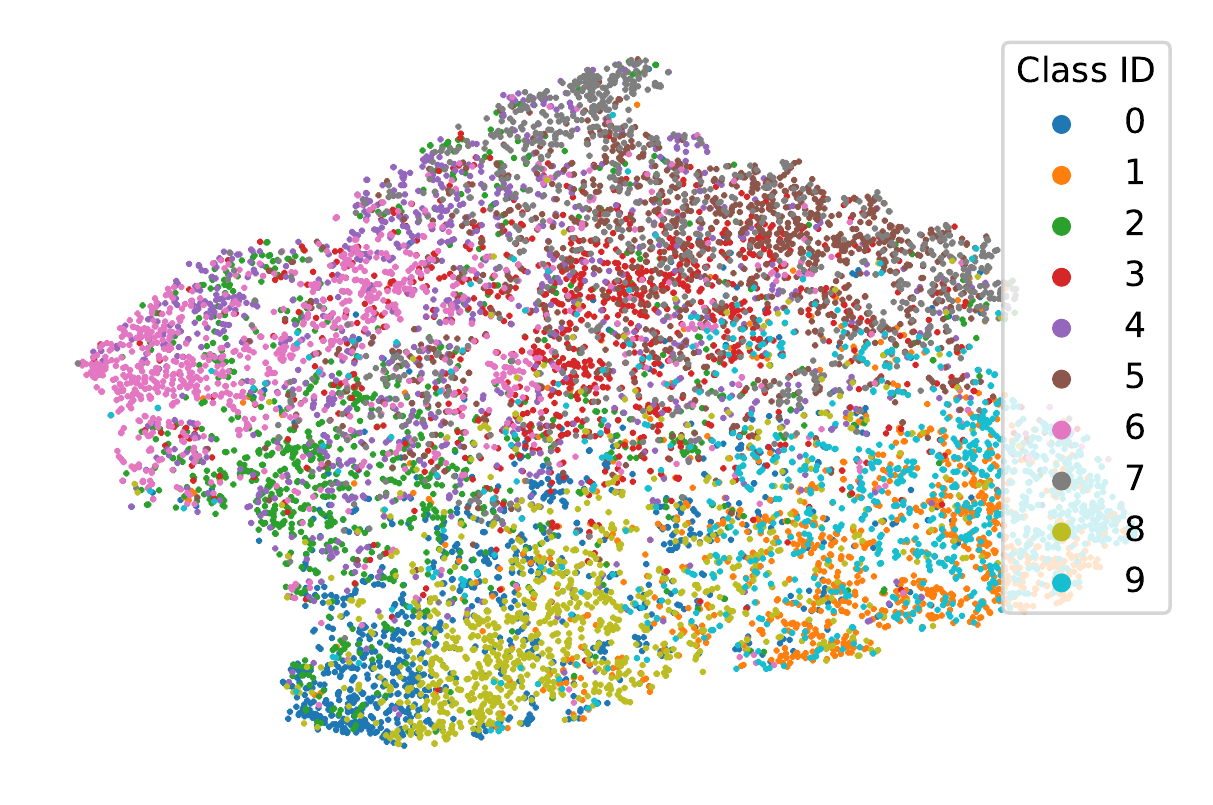}} \ \ \
\subfloat[Local Model of FedProx]{\label{fig:cifar10-tsne_prox_client}
\includegraphics[width=1.72in]{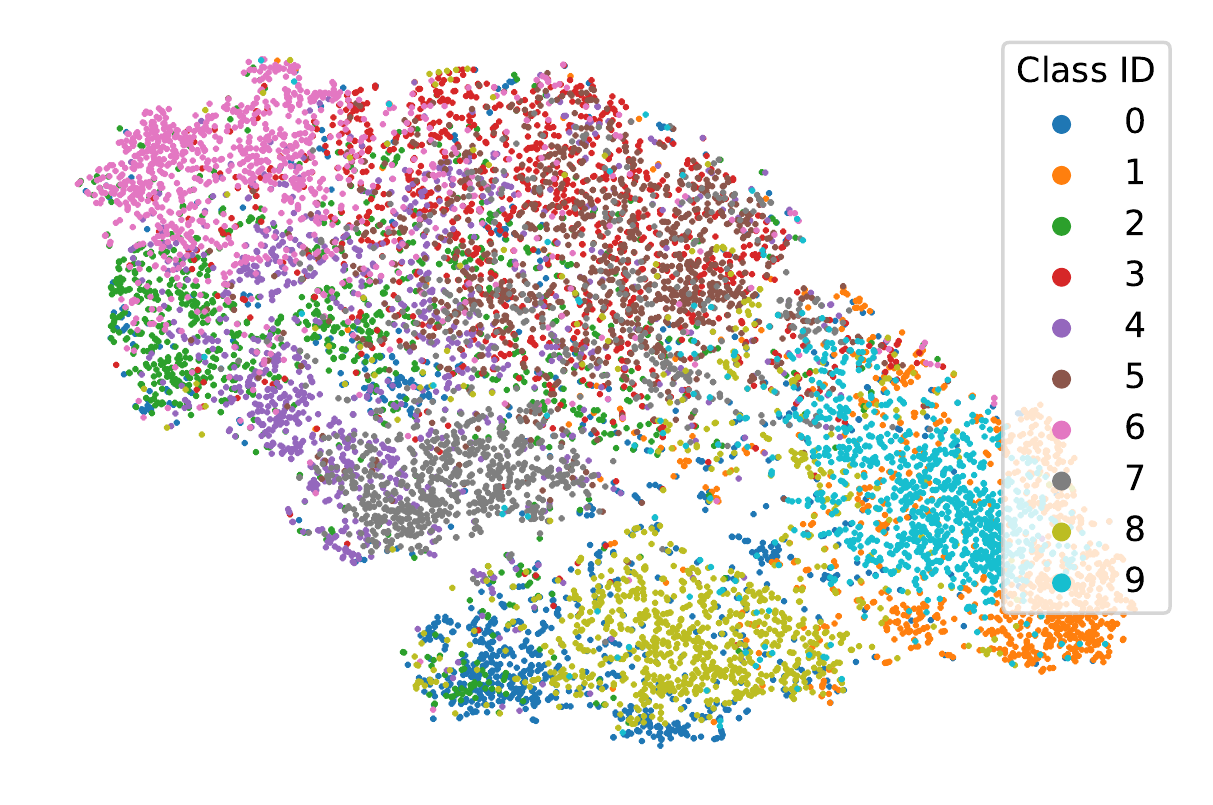}} \ \ \
\subfloat[Local Model of FedGKD]{\label{fig:cifar10-tsne_gkd_client}
\includegraphics[width=1.72in]{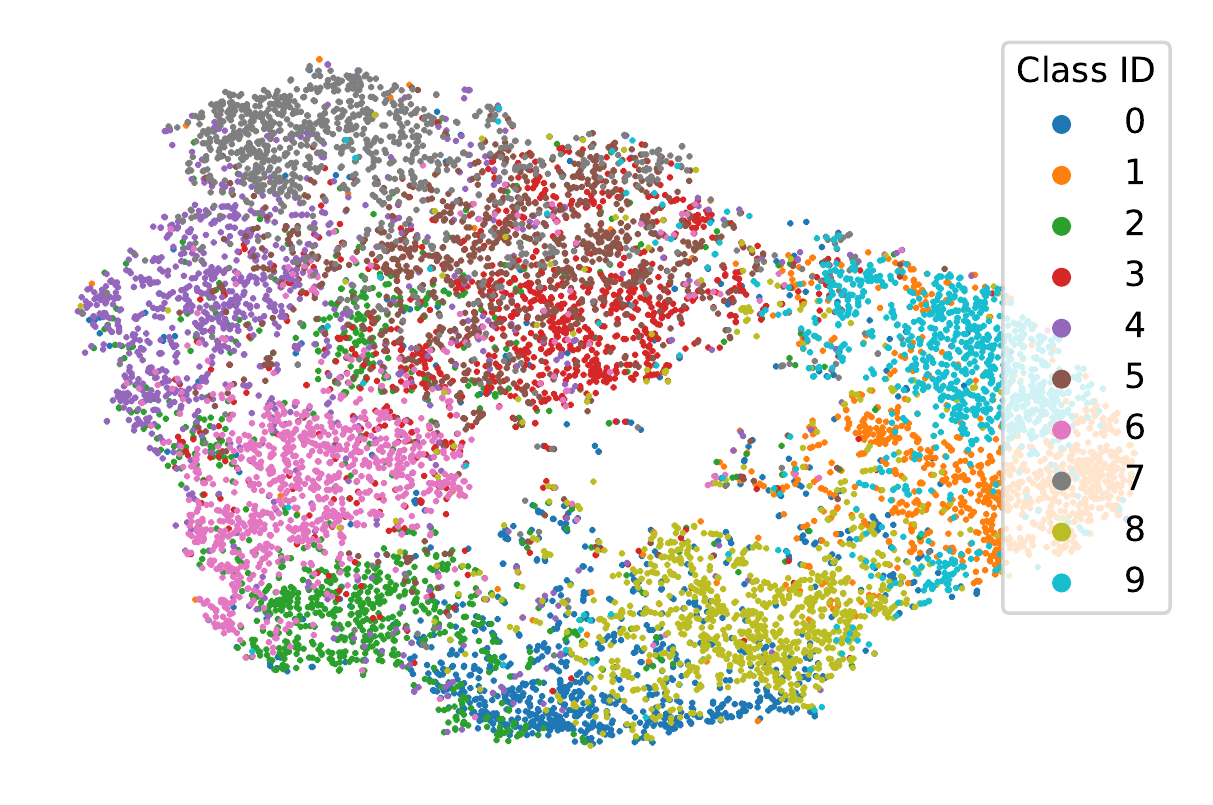}} \ \ \
\caption{T-SNE visualizations of feature of model on CIFAR-10 test dataset (Model of 100 communication rounds under $\alpha$=0.1 setting). \textbf{The global model accuracy of FedAvg, FedProx and \system is 39.99\%, 64.08\% and 65.98\%, test accuracy of the local model is 16.02\%, 32.11\% and 44.73\%, respectively.}} \label{fig:cifar10-tsne}
\end{figure*}

\section{Experiments}

\subsection{Experimental Setup}
\begin{figure*}[htb]
\centering
\subfloat[$\alpha$=1]{
\begin{minipage}[t]{0.32\linewidth}
\centering
\includegraphics[width=1.8in]{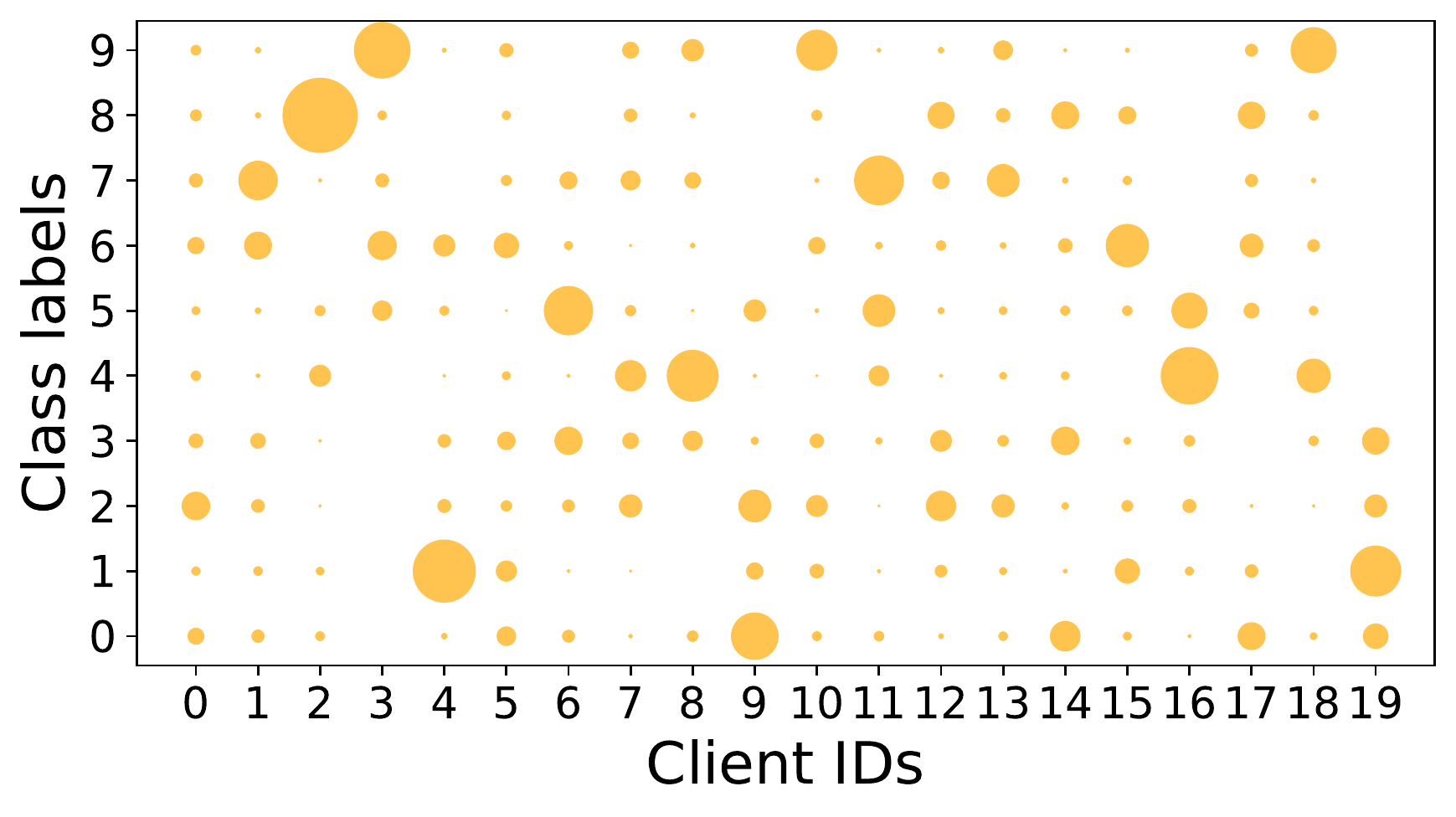}  \\
\includegraphics[width=1.8in]{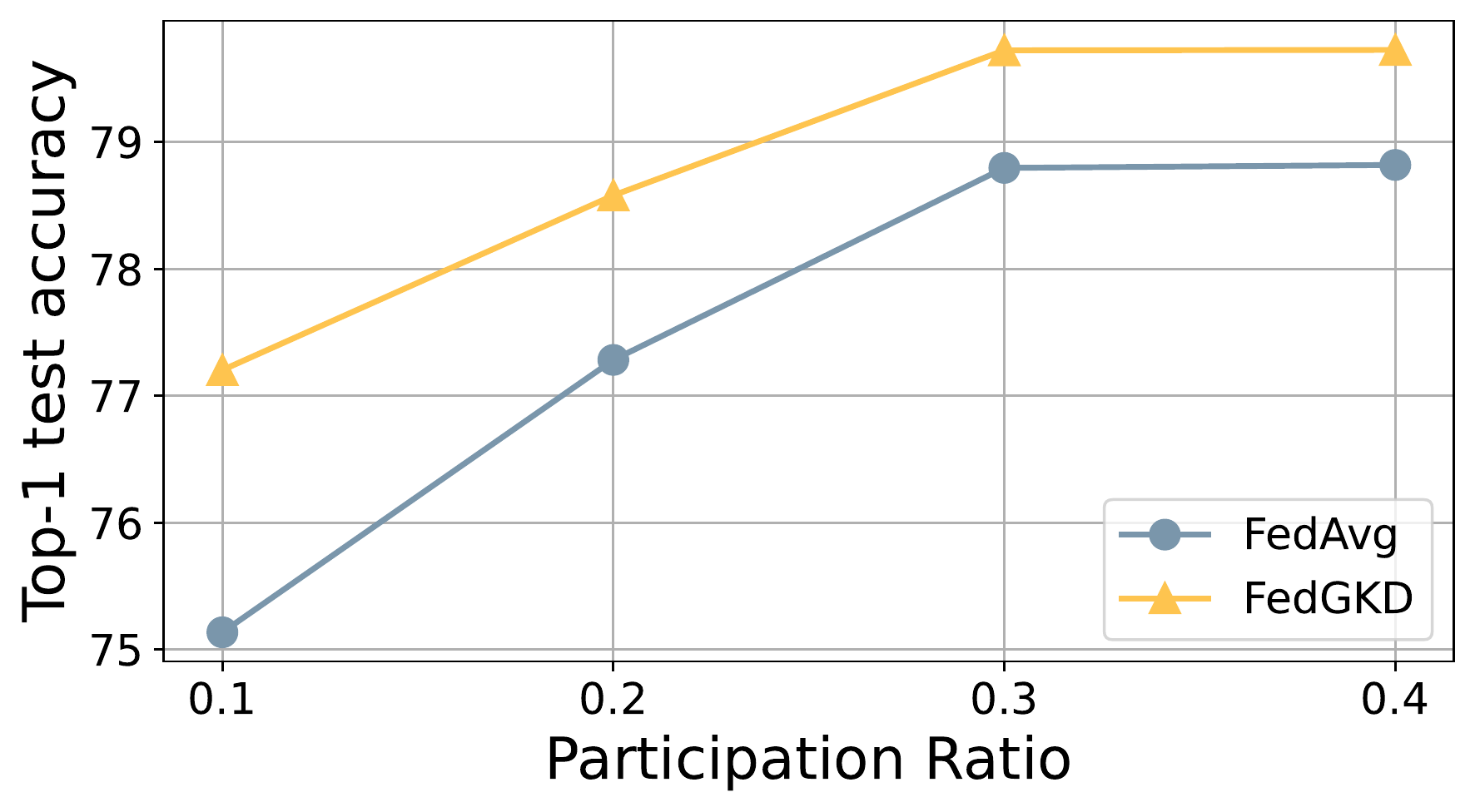}
\end{minipage}%
}
\subfloat[$\alpha$=0.5]{
\begin{minipage}[t]{0.32\linewidth}
\centering
\includegraphics[width=1.8in]{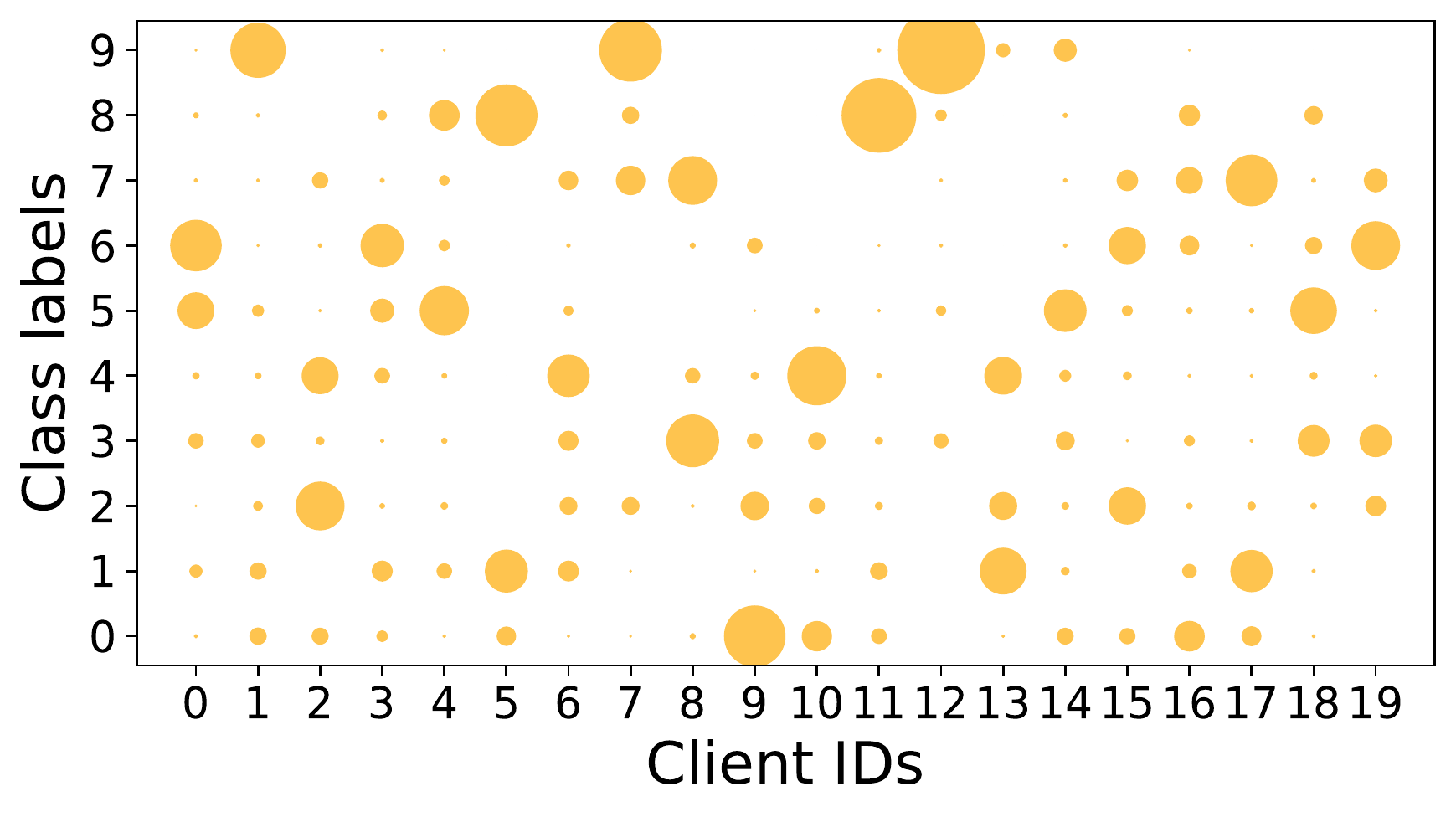}\\
\includegraphics[width=1.8in]{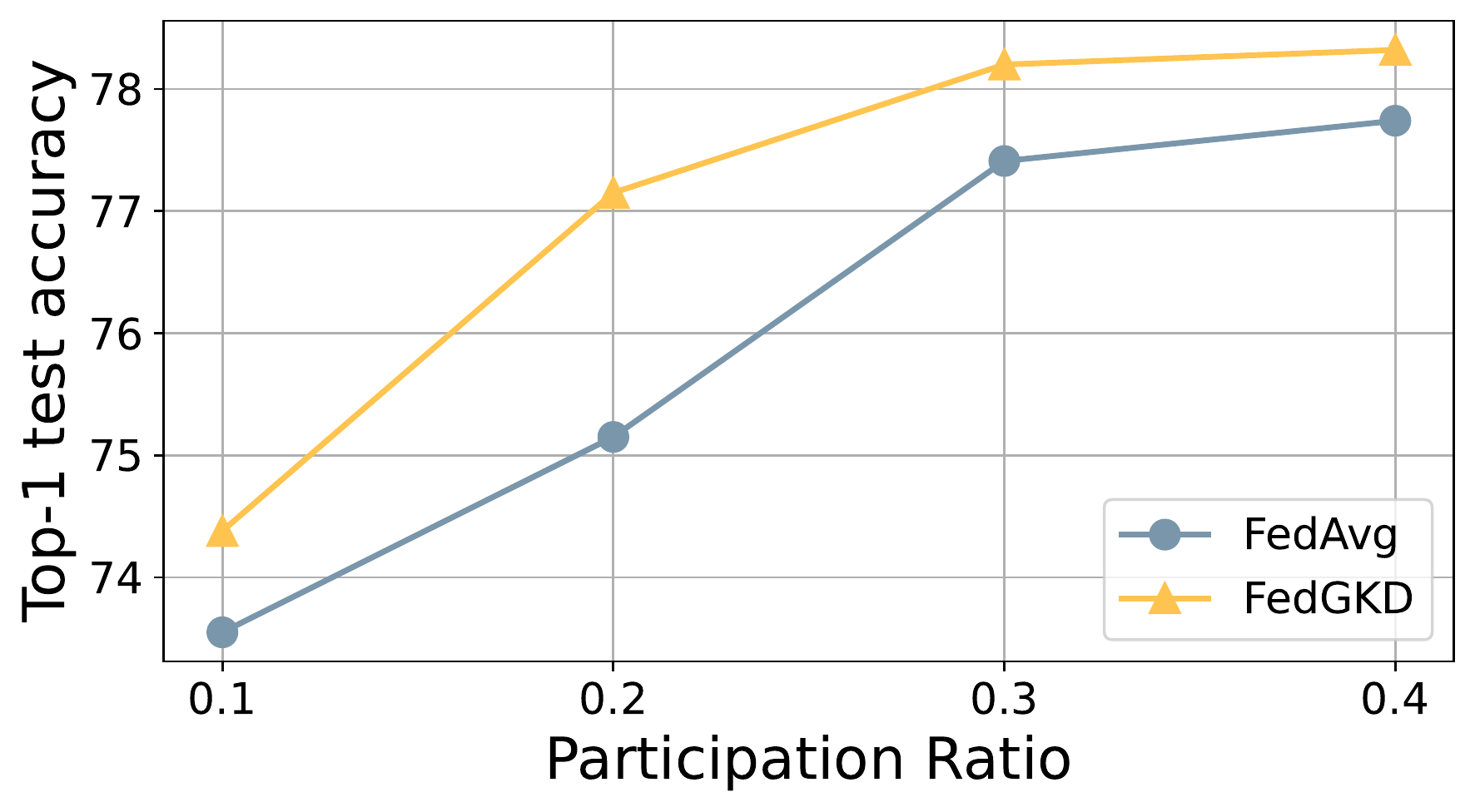}
\end{minipage}%
}
\subfloat[$\alpha$=0.1]{
\begin{minipage}[t]{0.32\linewidth}
\centering
\includegraphics[width=1.8in]{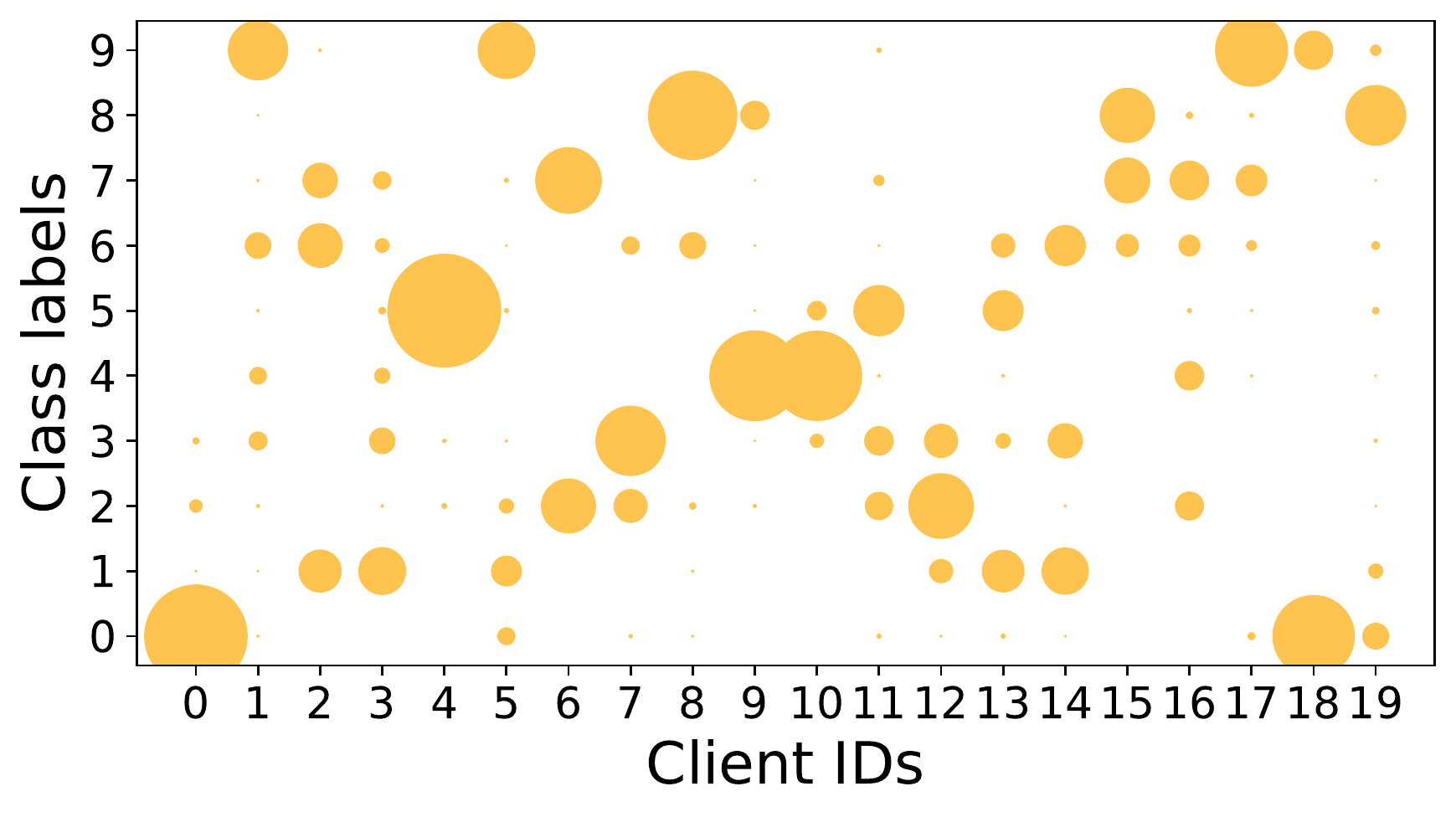}\\
\includegraphics[width=1.8in]{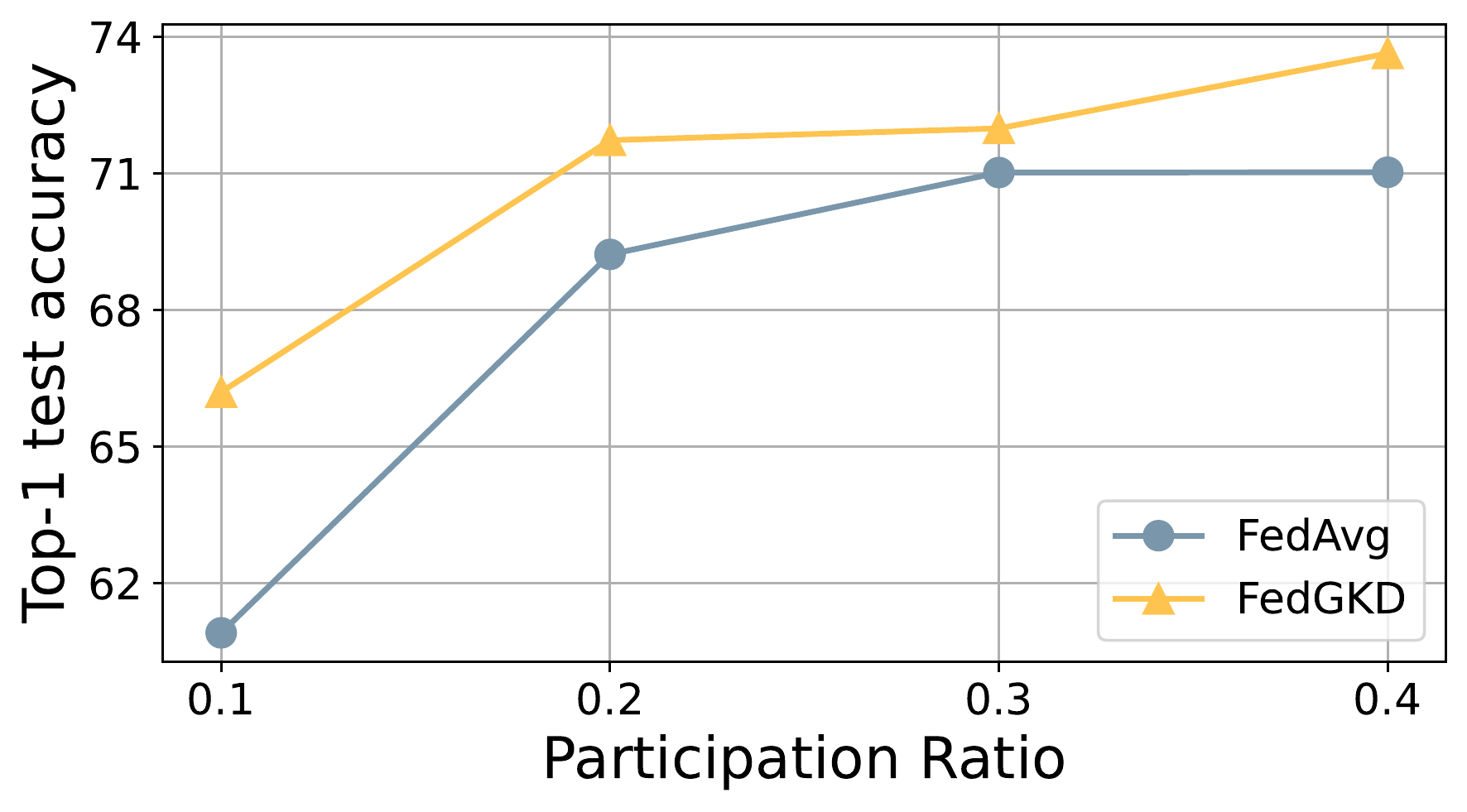}
\end{minipage}%
} \ \ \

\caption{\textbf{Top: Visualization of the non-IID ness} on CIFAR-10 dataset. \textbf{Bottom: Performance of FedAvg and \system(M=1)} with ResNet-8 on CIFAR-10 under different $\alpha$ and participation ratio setting.} \label{fig:num_class}
\end{figure*}
\paragraph{Datasets}

We compare with previous FL methods on both CV and NLP tasks and use architectures of ResNet~\citep{cvpr/HeZRS16} and DistillBERT~\citep{sanh2019distilbert} respectively. For CIFAR-10 and CIFAR-100~\citep{krizhevsky2009learning}, we conduct image classification tasks using ResNet-8. For Tiny-ImageNet dataset, we train the ResNet-50. For all these CV tasks, we use 90\% of the data as the training dataset. For NLP tasks, we perform text classification tasks on a 5-class sentiment classification dataset SST-5 and a 4-class news classification dataset AG News. We use 50\% and 100\% of the data respectively as the  training datasets. The datasets correspond to each task are summarized in Tab.~\ref{tbl:table-data}.




\begin{table}[!ht]
\centering
 \caption{Statistics of training datasets.}
\small
\setlength{\tabcolsep}{4pt} 
\begin{tabular}{llcccc}
\hline
  Dataset & Task & Train Data Size & Class & clients\\
  \hline
  CIFAR-10& image& 45,000 & 10 & 20\\
  CIFAR-100& image& 45,000& 100 & 20\\
  Tiny-Imagenet & image& 90,000 & 200 & 20\\
  \hline
  SST-5& text& 8,544 & 5 & 10 \\
  AG News& text& 60,000& 4 & 20\\
    \bottomrule
  \end{tabular}
 \label{tbl:table-data}
\end{table}
Following prior work~\citep{lin2020ensemble,hsu2019measuring}, we use the Dirichlet distribution $\bf{Dir}(\alpha)$ to sample disjoint non-IID client training data, where $\alpha$ denotes the concentration parameter, and the smaller $\alpha$ indicates higher data heterogeneity. Fig.~\ref{fig:num_class} visualizes the sampling results on CIFAR-10 in our experiments. 


\paragraph{Implementation Details}
The proposed \system and baselines are all implemented in PyTorch and evaluated on a Linux server with four V100 GPUs, MPI is used as communication backend to support distributed computing.

Considering Batch Normalization(BN) fails on heterogeneous training data~\citep{hsieh2020non} due to the statistics of running mean and variance for the local data, we replace BN by Group Normalization(GN) to produce stabler results and alleviate some of the quality loss brought by BN. It is worth noting that on tiny-imagenet, if GN is used in ResNet-50 and a 2-layer MLP is added, the model will not converge. To compare with MOON~\citep{li2021model}, we use the BN on ResNet-50 in MOON and \system$^{+}$.

For all the CV tasks, we set the channels of each group as 16. The optimizer of client training used is SGD, we tune the learning rate in $\{0.1,0.05,0.01\}$ for  FedAvg and set the learning rate as 0.05, weight decay is set $1e\text{-}5$, momentum is 0.9. 
We run 100 global communication rounds on CIFAR-10/100 datasets, and 30 communication rounds on Tiny-Imagenet dataset, the local training epochs is set as 20, and the default participant ratio is set as C=0.2. 

For all the NLP fine-tuning tasks, Adam is used as local optimizer, learning rate is set $1e\text{-}5$, weight decay is set 0.
Considering the fine-tuning process could be done in a few epochs, the communication round is set 10, local epoch is set 1 on AG News, 3 on SST-5. The clients participated each round is set 4.

The training batch size on all the tasks is set B=64.We run all the methods three trials and report the mean and standard derivation.

\paragraph{Parameter Setting}

As stated in MOON~\citep{li2021model}, the projection head is necessary for the method, for fair comparison, we add the experiments with projection head.The projection layer setting is same as SimCLR~\citep{chen2020simple}, 2-layer MLP with output dimension 256.To compare with MOON, we set the coefficient $\mu$=5,5,1,0.1 on CIFAR-10,CIFAR-100, Tiny-imagenet and NLP datasets.As for the temperature, we set it 0.5 as reported in ~\citep{li2021model}.

To compare with FedProx, we set $\mu$=0.01,0.001,0.001,0.001 on CIFAR-10,CIFAR-100, Tiny-imagenet and NLP datasets.

The distillation temperature of FedDistill$^{+}$, and for the distillation coefficient, we set it as 0.1 for all the datasets.

To compare with FedGen, we follow the global generator training setting on server as ~\citep{zhu2021data}, and tune the local regularization coefficient from $\{0.1, 1,10\}$. The hidden dimension of 2-layer MLP generator is set 512,2048,768 for ResNet-8, ResNet-50, DistilBERT.

For \system, we set $\gamma=0.2$ for ResNet-8 and DistilBERT, $\gamma=0.1$ for ResNet-50. For \systemled, the $\gamma_1\ldots\gamma_M$ are determined by validation performance.

For \systemled, we set the default buffer size as default 5 on CV tasks, 3 on NLP tasks, the parameters $\gamma_1\ldots\gamma_M$ are set by validation performance follows the calculation: $\frac{\gamma_i}{2}=\lambda \frac{e^{-\frac{L_i}{\beta}}}{\sum_i^M{e^{-\frac{L_i}{\beta}}}}$, where $L_i$ is the validation loss of the $i$th global model in the buffer, $\beta$ denotes the temperature and $\lambda$ is the coefficient to control the scale, we set $\beta$ as $\frac{1}{M}$ for simplicity and $\lambda$ is set 0.1.  


\paragraph{Evaluation Methods}
We compare our proposed \system with the following state-of-the-art approaches.
\begin{itemize}
  \item \textbf{FedAvg}: FedAvg~\citep{mcmahan2017communication} is a classic federated learning algorithm, which directly uses averaging as aggregation method.
  \item \textbf{FedProx}: FedProx~\citep{li2018federated} improves the local objective based on FedAvg. It adds the regularization of proximal term for local training.
  \item \textbf{MOON}: MOON~\citep{li2021model} learns from contrastive learning~\citep{chen2020simple}, injecting the projection layer for the original model. The local previous model features are seen as negative and global model features are seen positive samples.
  \item \textbf{FedDistill$^{+}$}: FedDistill~\citep{seo2020federated} shares the logits of user-data and adds the distillation regularization term for clients. We further develop it with model parameters sharing and name it FedDistill$^{+}$.
  \item \textbf{FedGen}: FedGen~\citep{zhu2021data} shares the local label count information to train a global generator and regularizes the local training with this light-weighted generator.
  \item \textbf{\system}: The proposed method in Fig.~\ref{fig:framework}, sends the ensemble of global models to participants and updates local weights via optimizing Eq.~\ref{eq:obj_avg}.
  \item \textbf{\systemled}: The server sends $M$ models to participants and updates local weights via optimizing Eq.~\ref{eq:obj_his}.
  \item \textbf{\systemplus}: To compare with MOON, we add the projection layer for \system and develop it as \systemplus.
\end{itemize}

\begin{table*}[!htb]
\renewcommand\arraystretch{1.1}
\centering
 \caption{Top-1 Test Accuracy overview given different data settings for ResNet on CV Datasets(20 clients with C=0.2, 100 communication rounds, buffer length for \systemled is set 5). \textbf{Bold} number denotes the best result and \underline{underline} denotes the second best result.}
\resizebox{1.\linewidth}{!}{
\begin{tabular}{llcccccccc}
\hline
\multicolumn{1}{c}{Dataset} & \multicolumn{1}{c}{Setting} & FedAvg & FedProx & MOON  & FedDistill$^{+}$ & FedGen &  \system & \systemled & \systemplus \\ 
  
  \hline
  \multirow{3}{*}{CIFAR-10} &  $\alpha$=1 &77.30$\pm$0.72 & 76.39$\pm$0.59 & 78.17$\pm$0.81 & 77.79$\pm$0.51 &76.44$\pm$0.94 & \underline{79.14$\pm$0.72} & 78.88$\pm$0.39 & \bf{79.45$\pm$0.37} \\
  &$\alpha$=0.5 & 75.15$\pm$0.44 & 75.02$\pm$0.74 & 77.24$\pm$0.90 & 75.85$\pm$0.63 & 73.62$\pm$0.20 &77.15$\pm$0.28 & \bf{77.44$\pm$0.60} & \underline{77.38$\pm$0.60}\\
  &$\alpha$=0.1 &69.22$\pm$0.91 & 68.79$\pm$1.10 & 71.12$\pm$0.69 & 67.23$\pm$0.33 & 69.24$\pm$0.32 & 72.27$\pm$0.84 & \underline{72.99$\pm$0.49} & \bf{73.06$\pm$0.32} \\
  \hline
  \multirow{3}{*}{CIFAR-100} & $\alpha$=1 & 42.22$\pm$0.41 & 41.67$\pm$0.56 & 40.76$\pm$0.04 & 42.38$\pm$0.24 & 39.83$\pm$0.58 & \bf{43.82$\pm$1.39} & \underline{42.91$\pm$0.23} & 40.98$\pm$0.41 \\
  &$\alpha$=0.5 & 40.42$\pm$0.58 & 40.08$\pm$0.24 & 35.89 $\pm$1.12 & 40.97$\pm$0.12 & 39.30$\pm$0.64 & \bf{42.44$\pm$0.32}& \underline{42.34$\pm$0.31} & 37.91$\pm$0.63 \\
  &$\alpha$=0.1 & 34.91$\pm$1.00 & 35.80$\pm$0.15 & 31.90$\pm$0.51 & 33.94 $\pm$0.45 & 33.13$\pm$0.43 & \underline{36.83$\pm$0.20}&\bf{37.29$\pm$0.07} & 33.42$\pm$0.17  \\
  \hline
  \multirow{3}{*}{Tiny-ImageNet}
  &$\alpha$=1 &34.81$\pm$1.57 & 36.54$\pm$1.67 & \underline{43.20$\pm$0.66} & 35.99$\pm$1.18 & 34.79$\pm$1.15 & 37.50$\pm$0.43 &35.79$\pm$1.04 & \bf{43.78$\pm$1.18}\\
  &$\alpha$=0.5 & 32.48$\pm$1.59 & 35.41$\pm$0.64 & \underline{40.16$\pm$1.96} & 34.02$\pm$0.73 & 34.47$\pm$0.63 & 37.28$\pm$0.86 & 37.35$\pm$0.65 & \bf{40.56$\pm$2.34} \\
  &$\alpha$=0.1 & 29.02$\pm$0.65 & 31.93$\pm$0.85 & 33.56$\pm$0.16 & 28.95$\pm$0.10 & 32.84$\pm$0.99  & \underline{34.64$\pm$0.21} &33.71$\pm$1.74 & \bf{35.72$\pm$0.71} \\
  \bottomrule

  \end{tabular}
  }
 \label{tbl:table_compare}
\end{table*}

\begin{table*}[!htb]
\renewcommand\arraystretch{1.1}
\centering
 \caption{Performance overview given different data settings for DistilBERT (4 participants each round, 10 communication rounds, buffer length for \systemled is set 3).}
\resizebox{1.\linewidth}{!}{
\begin{tabular}{llcccccccc}
\hline
\multicolumn{1}{c}{Dataset} & \multicolumn{1}{c}{Setting} & FedAvg & FedProx & MOON  & FedDistill$^{+}$ & FedGen & \system & \systemled & \systemplus \\ 
  
  \hline
  AG News &  $\alpha$=0.1 & 89.19$\pm$0.87 & 89.19$\pm$0.87 & 89.25$\pm$0.48 & 89.35$\pm$0.85 &89.28$\pm$0.77 & \bf{89.60$\pm$0.44}  &89.04$\pm$0.22 & \underline{89.48$\pm$0.18} \\
  \hline
  SST-5 &  $\alpha$=0.1 & 40.39$\pm$0.50 & 40.42$\pm$0.49 & 41.48$\pm$2.24 & 41.25$\pm$1.03 & 41.00$\pm$3.23  & \bf{43.24$\pm$1.80} &\underline{41.93$\pm$2.05} & 41.21$\pm$1.20\\
  \bottomrule
  \end{tabular}
  }
 \label{tbl:table_nlp}
\end{table*}
\subsection{Experimental Results}
\paragraph{Results Overview}
In Fig.~\ref{fig:num_class}, we give a brief comparison for FedAvg and \system under various participation ratio setting, \system constantly outperforms FedAvg. 
We compare the  top-1 test accuracy of the different methods as reported in Tab.~\ref{tbl:table_compare} and Tab.~\ref{tbl:table_nlp}, the proposed method \system outperforms previous state-of-the-art methods 
Under the $\alpha=0.1$ (the hyper-parameter for Dirichlet distribution) setting, \system can outperform FedAvg by 5.6\% on Tiny-Imagenet, 1.9\% on CIFAR-100, 3.0\% on CIFAR-10, 0.4\% on AG News, and 2.8\% on SST-5. 
Moreover, we observe \system brings the smoothness for learning procedure compared with the FedAvg as shown in Fig.~\ref{fig:acc_plot}, and \system constantly outperforms FedAvg and FedProx in the training procedure. 
We present the advantage of \system to mitigate local drifts in Fig.~\ref{fig:cifar10-tsne}, the local model of FedAvg achieves only 16.02\% accuracy on CIFAR-10 test dataset, and the global model fails to classify the image features due to the local drifts, the global accuracy falls to 39.99\%. \system improves the test accuracy of the local model by 28.7\%, and over 25\% on global model test accuracy. FedProx improves the global accuracy to 64.08\% but the test accuracy of local model is 12.6\% lower than \system.

 \begin{table*}[!htb]
\renewcommand\arraystretch{1.1}
\centering
\caption{Comparison of different FL methods in different sampling ratio, we evaluate on CIFAR-10 with ResNet-8, $\alpha$=0.5.}
\resizebox{1.\linewidth}{!}{
\begin{tabular}{ccccccccc}
\hline
\multirow{2}{*}{Method}  & \multicolumn{2}{c}{C=0.1} &\multicolumn{2}{c}{C=0.2} & \multicolumn{2}{c}{C=0.3} &\multicolumn{2}{c}{C=0.4}  \\

\cline{2-9}
  & Best  & Final & Best & Final & Best  & Final & Best  & Final \\
  \hline
  FedAvg& 73.55$\pm$0.10& 73.31$\pm$0.35& 75.15$\pm$0.44 & 72.70$\pm$0.78 & 77.41$\pm$0.21& 74.18$\pm$0.21& 77.74$\pm$0.61& 74.45$\pm$0.52 \\
  \hline
  FedProx& 72.59$\pm$0.66& 72.59$\pm$0.66& 75.02$\pm$0.74& 72.08$\pm$0.92& 76.36$\pm$0.64& 74.55$\pm$0.48& 76.65$\pm$0.76& 74.03$\pm$0.61 \\
  \hline
  MOON& 73.47$\pm$0.47& 72.57$\pm$0.31& 77.24$\pm$0.90 & 74.48$\pm$1.01& 77.88$\pm$0.26& 75.00$\pm$0.33& 78.20$\pm$0.57& 75.05$\pm$0.44 \\
  \hline
  FedDistill$^{+}$& 73.75$\pm$0.35& 73.51$\pm$0.35& 75.85$\pm$0.63& 73.63$\pm$0.26& 77.31$\pm$0.65& 74.38$\pm$0.19& 78.22$\pm$0.45& 74.92$\pm$0.47\\
  \hline
  FedGen& 73.39$\pm$0.34& 73.39$\pm$0.34& 73.62$\pm$0.20& 72.27$\pm$0.42& 76.73$\pm$0.53& 73.75$\pm$0.35& 77.05$\pm$0.28& 73.96$\pm$0.58 \\
  \hline
  \system& 74.38$\pm$1.11& 74.10$\pm$0.86& 77.15$\pm$0.28& 74.27$\pm$0.33& \textbf{78.20$\pm$0.77}& \textbf{75.45$\pm$0.34}& \textbf{78.32$\pm$0.40}& \textbf{75.39$\pm$0.67} \\
  \systemled & \textbf{75.41$\pm$0.56} & 74.25$\pm$0.58 & \textbf{77.44$\pm$0.60} & 73.96$\pm$0.86&77.72$\pm$0.51 &75.74$\pm$0.31 & 78.08$\pm$0.58 & 75.75$\pm$0.75 \\
  \systemplus & 75.16$\pm$0.67& \textbf{74.46$\pm$0.54}& 77.38$\pm$0.60& \textbf{74.55$\pm$0.78}& 78.02$\pm$0.45& 74.54$\pm$0.42& 78.28$\pm$0.16& 75.24$\pm$0.33 \\
  \bottomrule
  \end{tabular}
  }
 \label{tbl:table1_complete}
\end{table*}

\begin{figure*}[!htb]
\centering
\subfloat[$\alpha$=0.1]{
\label{fig:learning_curve_alpha01}
\begin{minipage}[t]{0.31\linewidth}
\includegraphics[width=1.8in]{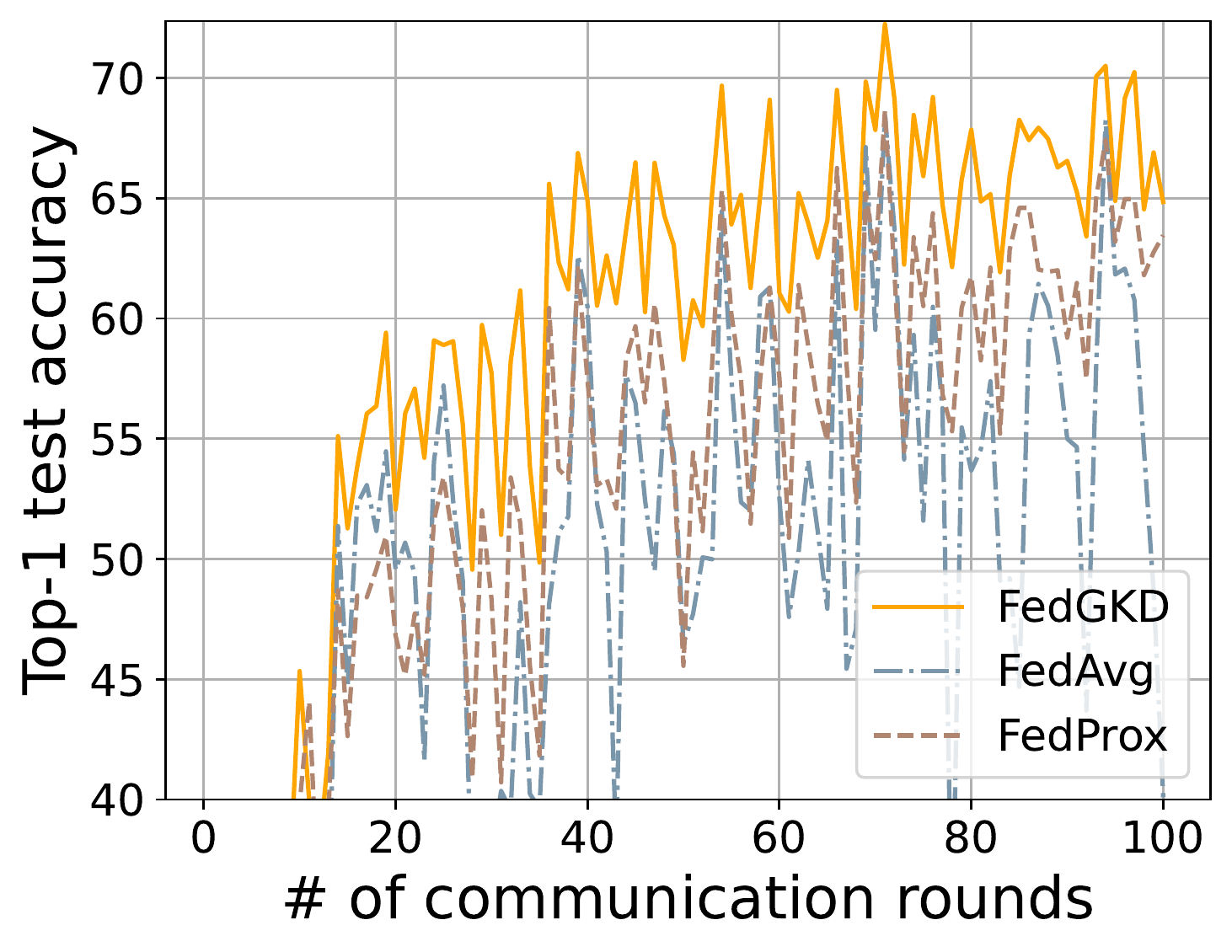}\ \ \
\includegraphics[width=1.8in]{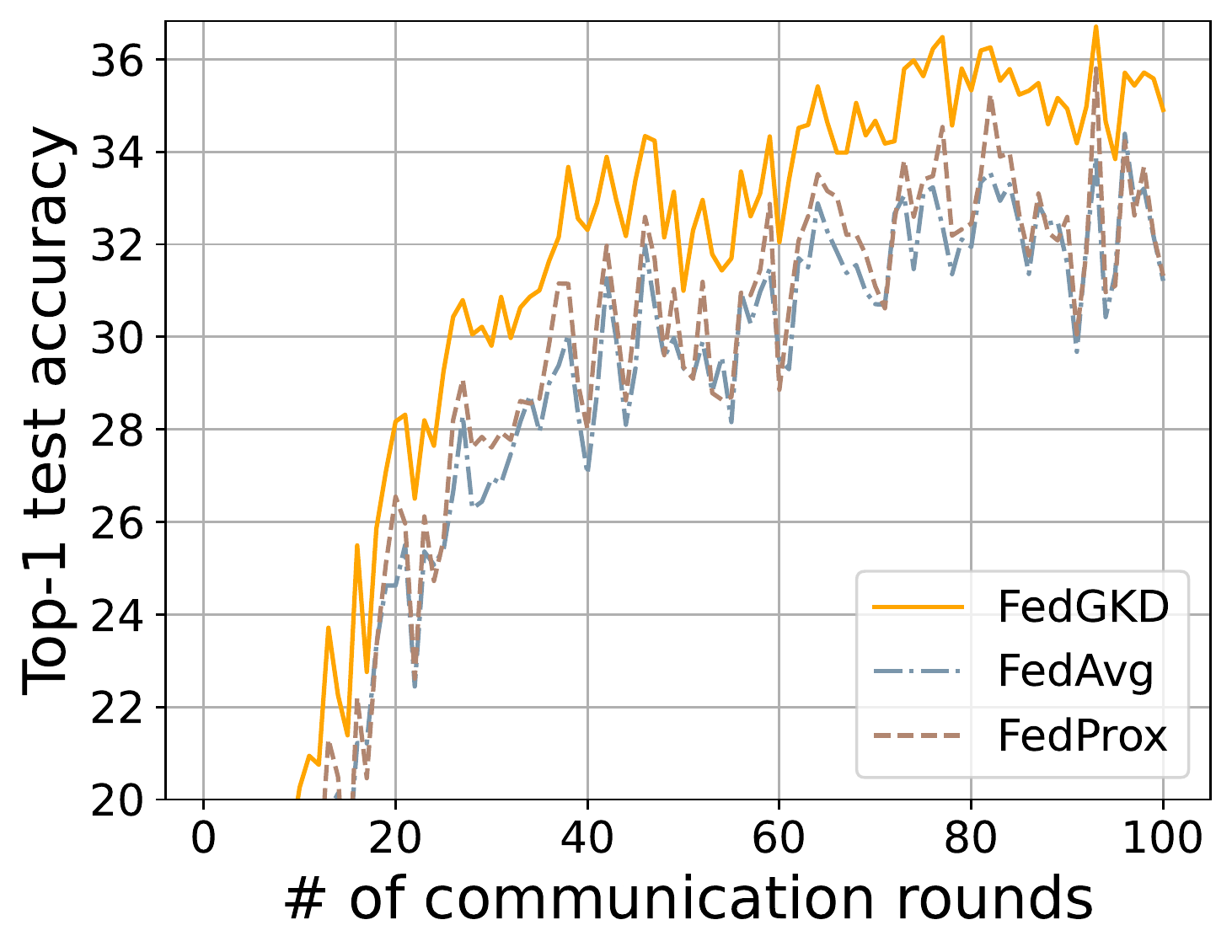}\ \ \
\includegraphics[width=1.8in]{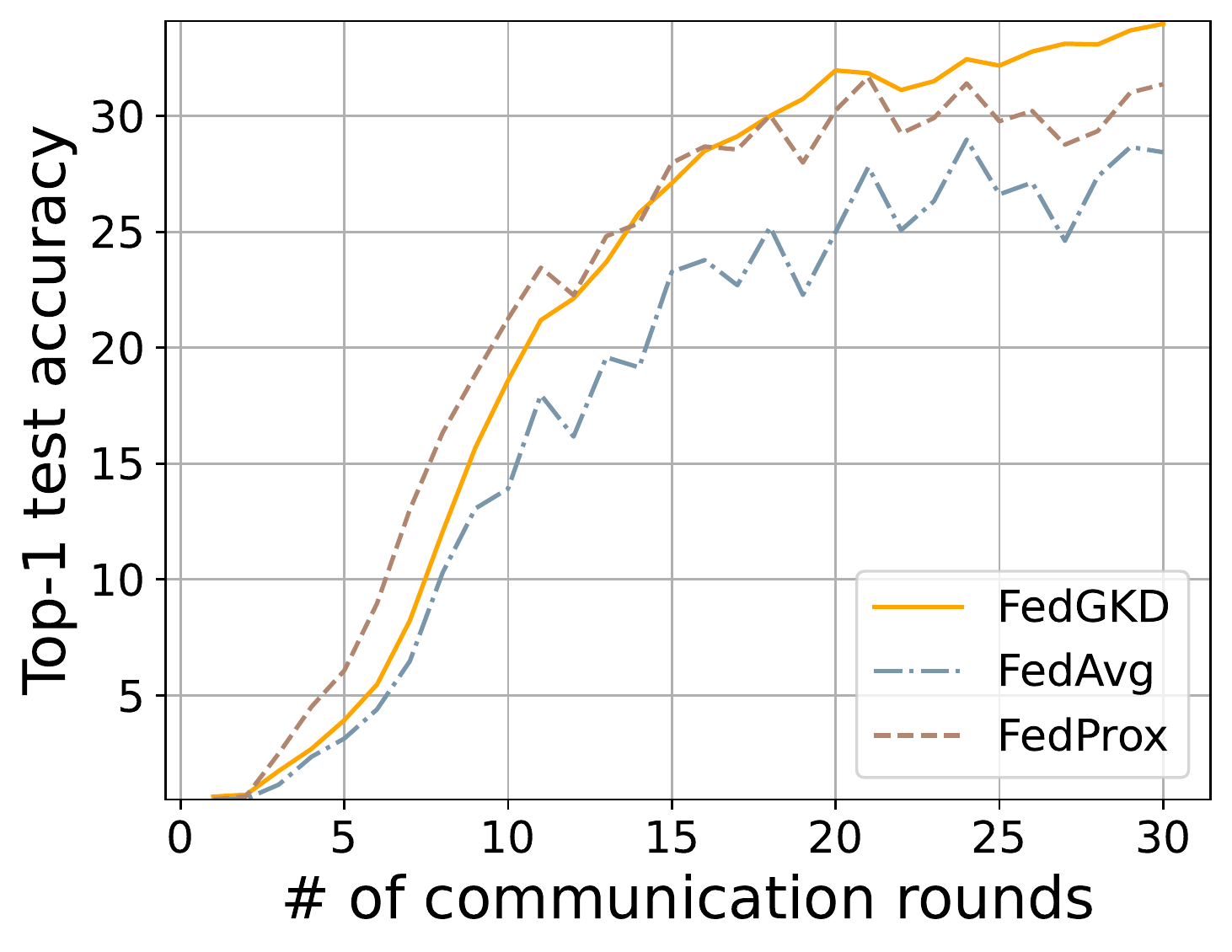}
\end{minipage}
}
\subfloat[$\alpha$=0.5]{
\label{fig:learning_curve_alpha05}
\begin{minipage}[t]{0.31\linewidth}
\includegraphics[width=1.8in]{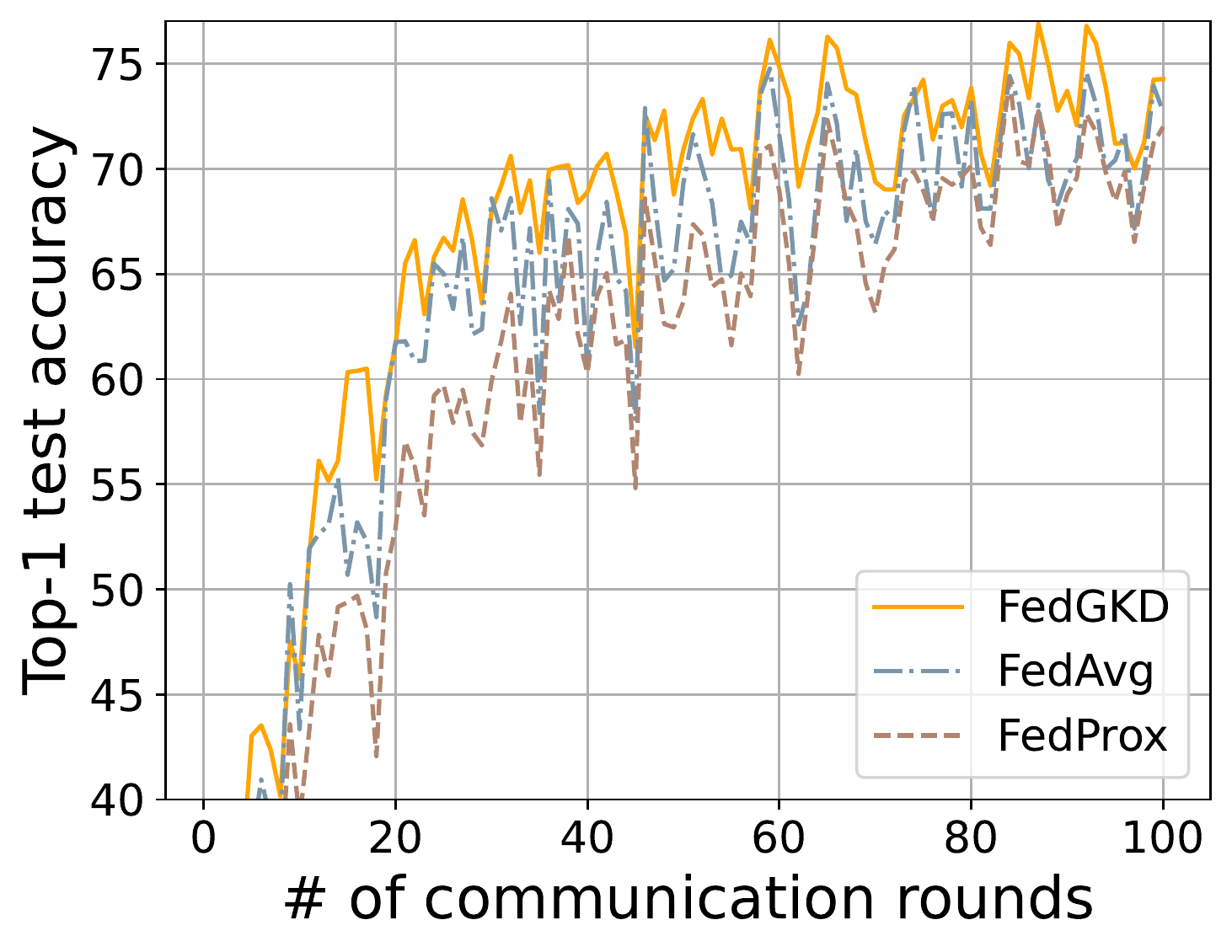}\ \ \
\includegraphics[width=1.8in]{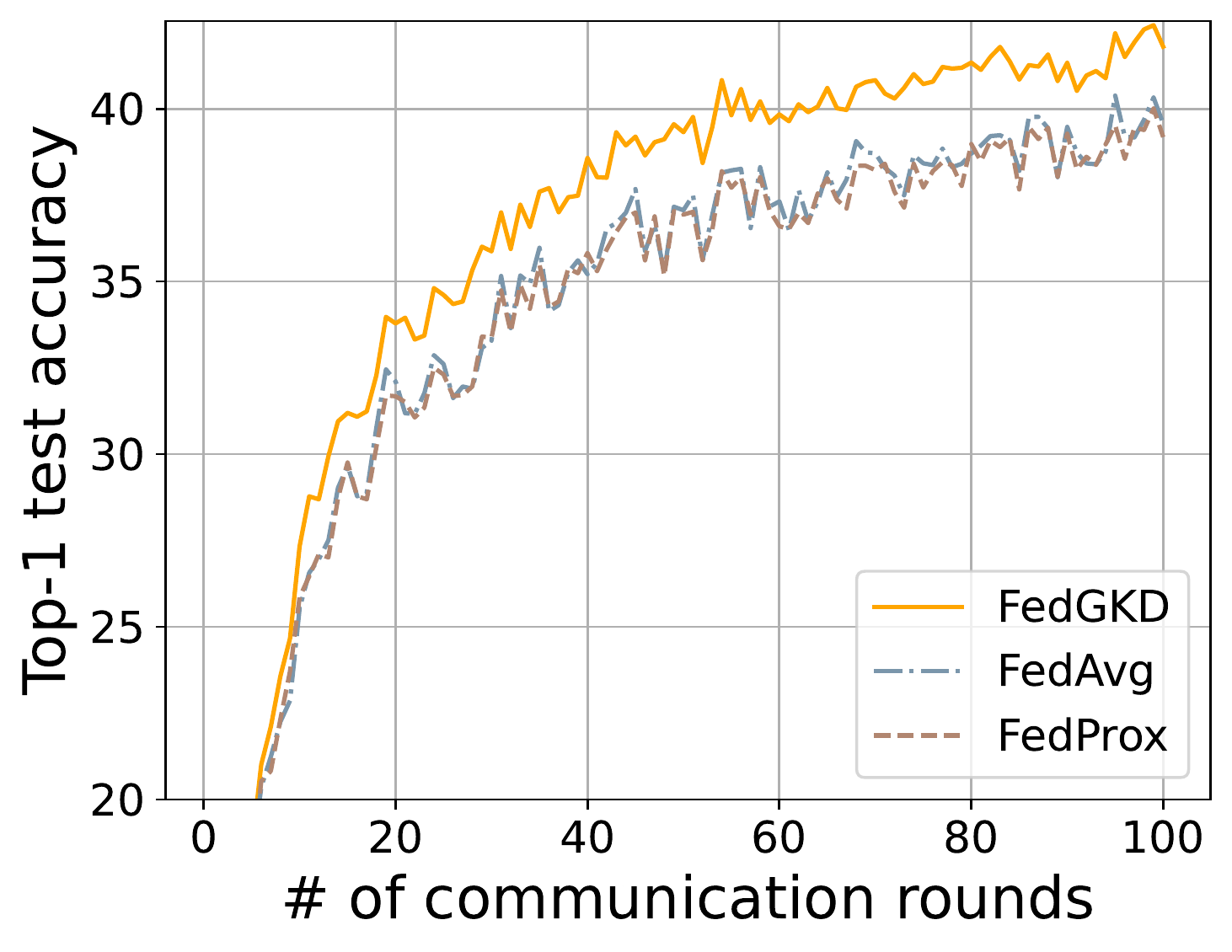}\ \ \
\includegraphics[width=1.8in]{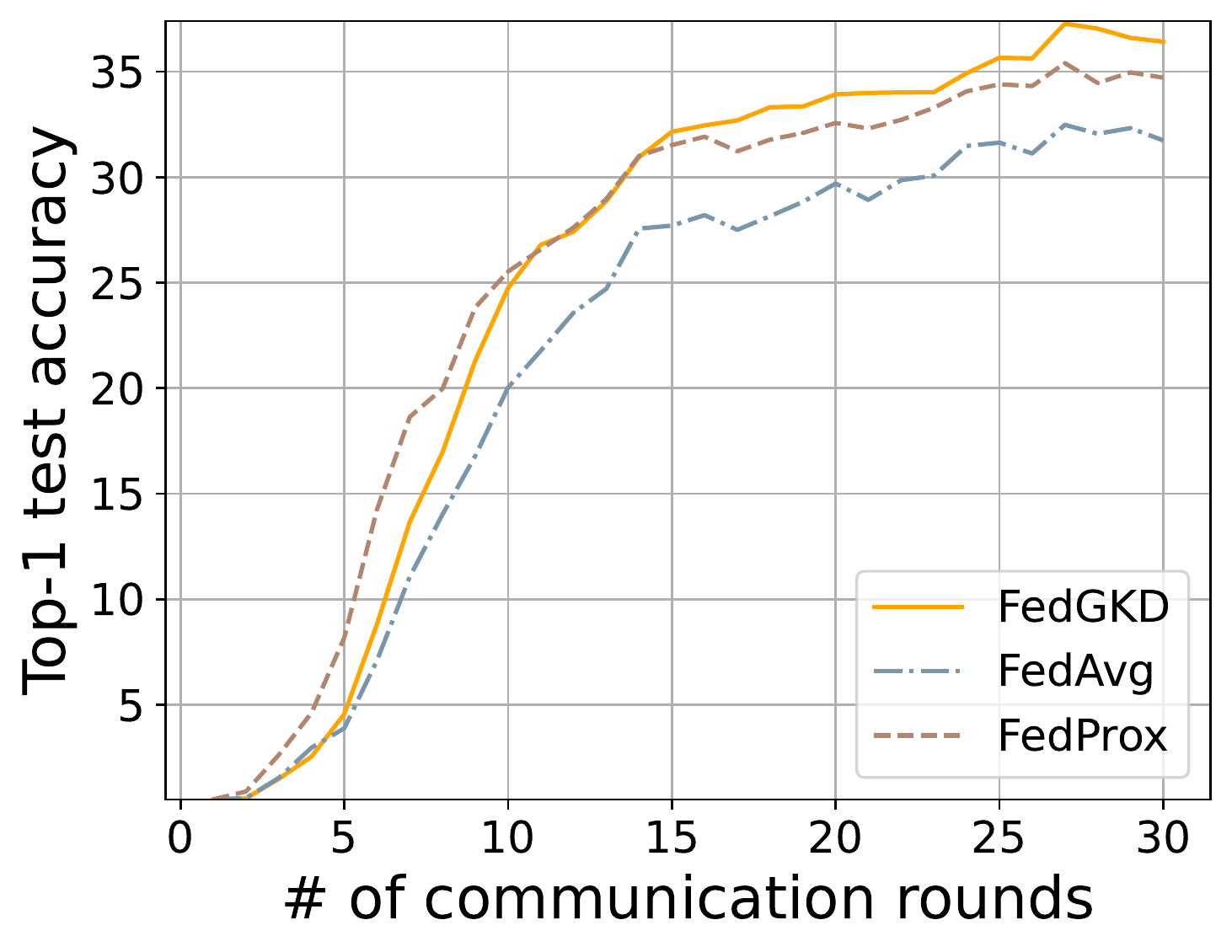}
\end{minipage}
}
\subfloat[$\alpha$=1]{
\label{fig:learning_curve_alpha1}
\begin{minipage}[t]{0.31\linewidth}
\centering
\includegraphics[width=1.8in]{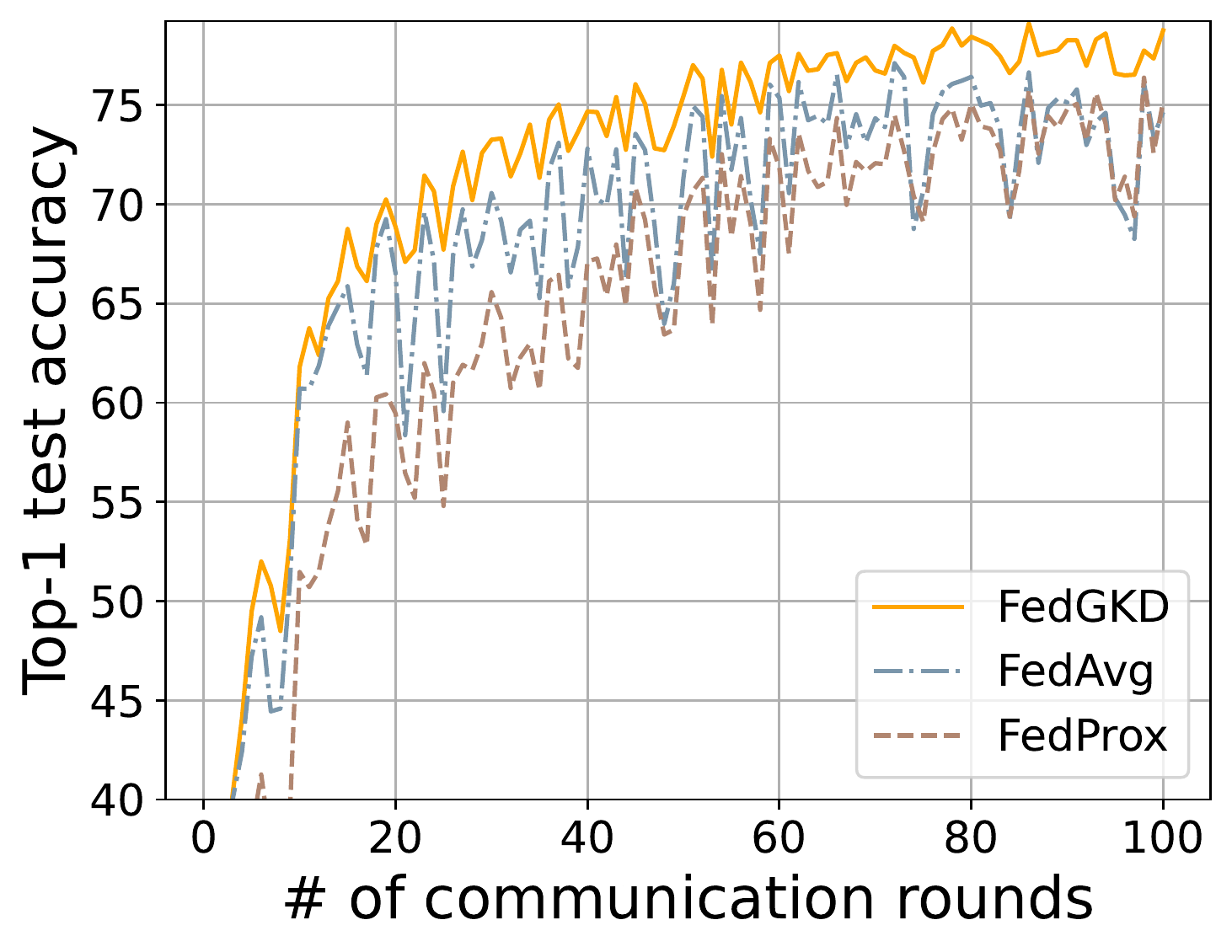}\ \ \
\includegraphics[width=1.8in]{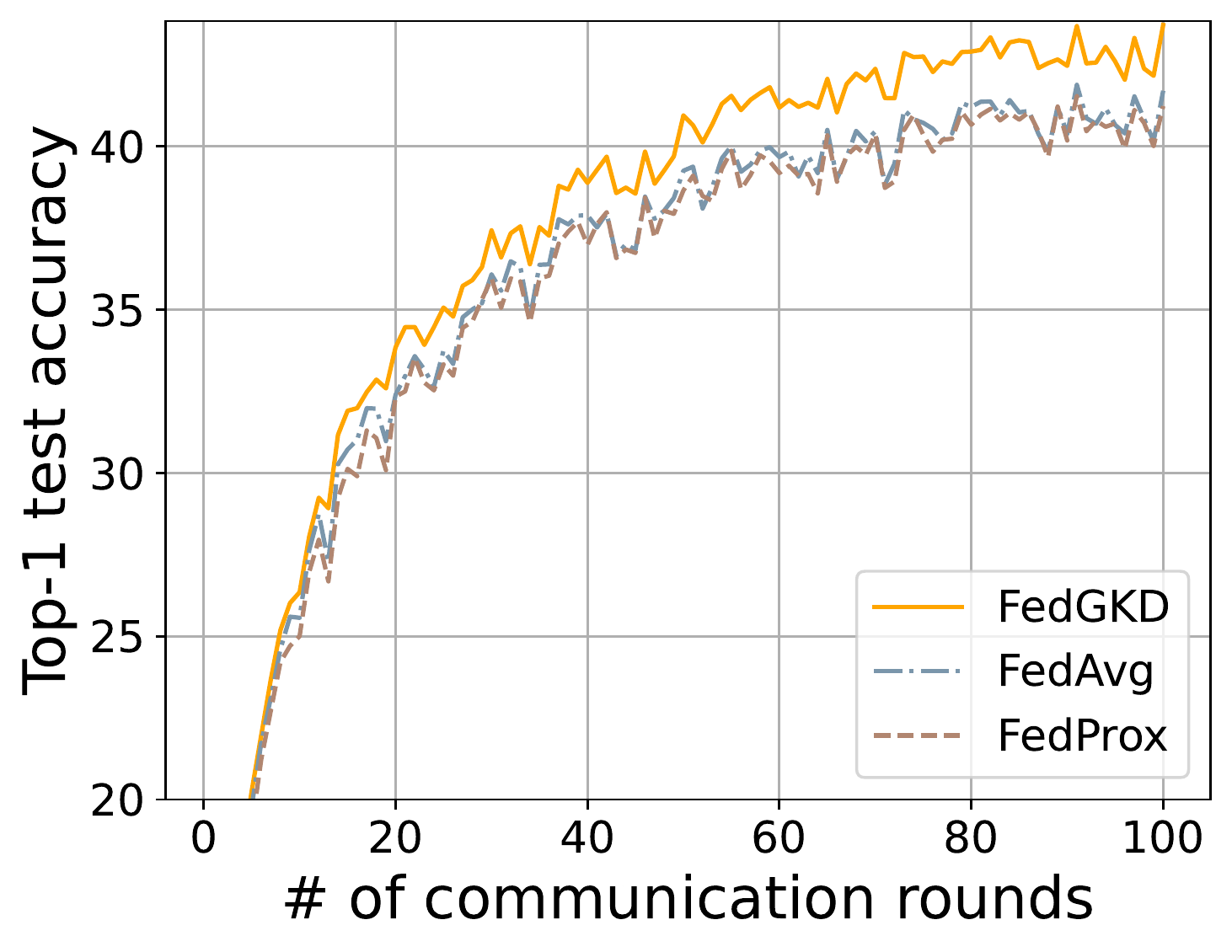} \ \
\includegraphics[width=1.8in]{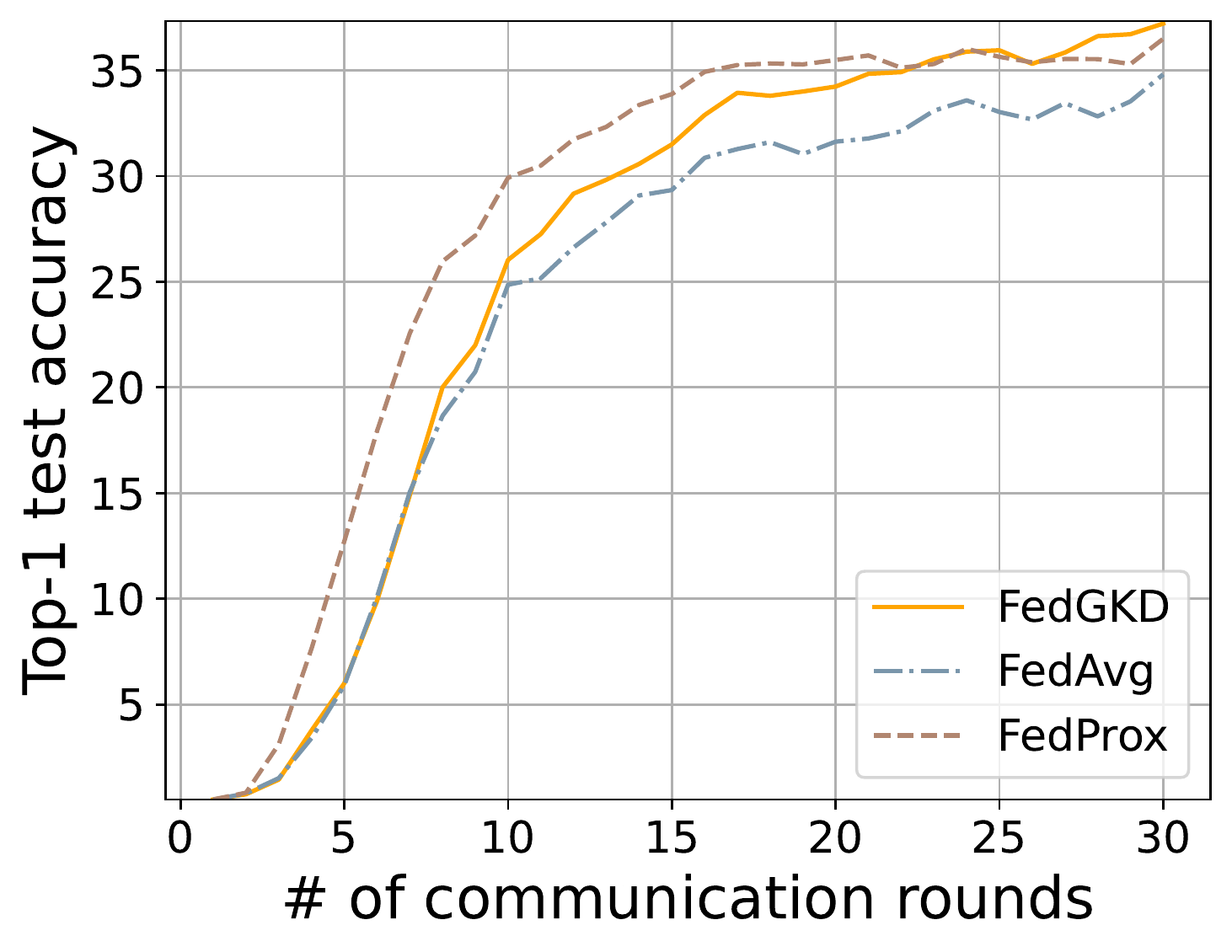}
\end{minipage}
}
\caption{Test accuracy of FedAvg, FedProx and \system on CV datasets in different number of communication rounds. \textbf{From top to bottom, the dataset is CIFAR-10, CIFAR-100, Tiny-ImageNet.} From left to right, the $\alpha$ ranges from 0.1 to 1.} \label{fig:acc_plot}
\end{figure*}


\paragraph{Impact of Participation Ratio}

As shown in Tab.~\ref{tbl:table1_complete}, we give a more comprehensive description to exhibit the advantage of our proposed method. With the participants number growing, our proposed methods constantly outperform the other methods, \system shows its greater superiority over other methods when there are fewer participants. For FedGen, it is not trivial to manage the generator training procedure, and it's even difficult to learn global features when the participation ratio is low. FedDistill$^{+}$ sends the local logits of each client directly to server, treating the averaged logits as global logits, while the quality of global logits also suffers from the low participation ratio. Compared with the model-modification methods, we find \systemplus consistently outperforms MOON.
\system and \systemled benefit from the historical global models ensemble and distillation to produce a more accurate result.

\paragraph{Impact of Data Heterogeneity}
As shown in Tab.~\ref{tbl:table_compare}, \system, \systemled and \systemplus outperform other methods and show their generality in different $\alpha$ setting. \system shows the superiority when the data heterogeneity increases. \systemplus benefits from the modification on model structure on CIFAR-10 and Tiny-Imagenet.
For FedProx, if the coefficient $\mu$ is set as a small number, the performance is close to FedAvg, otherwise it harms the convergence when the data heterogeneity decreases. While for \system, setting the regularization coefficient as 0.2 is helpful in most cases.
FedGen regularizes the local model mainly for the classifier layer based on the light-weighted generator, which limits its performance.

\paragraph{Robustness of the Methods}
As shown in Fig.~\ref{fig:acc_plot}, we observed the learning curve of \system constantly outperforms FedAvg and FedProx on CIFAR datasets. \system performs worse than FedProx on Tiny-Imagenet at the start and outperforms it when the communication rounds increases, and \system shows more smoothing learning curve compared with FedAvg.
We find that our proposed method shows robustness at different communication rounds as shown in Tab.~\ref{tbl:acc_straggler}. Most of the evaluation methods drops over 20\% accuracy on the last communication round compared to their best accuracy(shown in Tab.~\ref{tbl:table_compare}) except FedProx and \system.  While \system and \systemled show higher accuracy besides the stability compared to FedProx, as a detailed comparison provided in Fig.~\ref{fig:cifar10-tsne}. Compared to MOON, which also learns from global model, \systemplus shows a more robust training process.

\begin{table}[tb]
\renewcommand\arraystretch{1.1}
\centering
\small
\caption{ Evaluation of FL methods at different communication round on CIFAR-10 with C=0.2, $\alpha=0.1$.}
\begin{tabular}{ccccc}
\hline
\multirow{2}{*}{Method}  & \multicolumn{4}{c}{Test Accuracy at different communication round} \\

\cline{2-5}
   & 10 & 20 & 50 & 100  \\
  \hline
  
  FedAvg  &36.30$\pm$1.24 &49.54$\pm$0.99 &46.47$\pm$1.67 &40.09$\pm$2.29 \\
  \hline
  FedProx  &39.81$\pm$1.60 &46.82$\pm$0.96 &45.55$\pm$2.31 &63.45$\pm$0.40 \\
  \hline
  MOON   &33.22$\pm$1.63 &43.55$\pm$1.12 &45.10$\pm$4.21 &44.57$\pm$1.08  \\
  \hline
  FedDistill$^{+}$ &35.96$\pm$1.75 &44.70$\pm$3.07 &40.02$\pm$1.66 &43.76$\pm$5.13 \\
  \hline
  FedGen  &37.87$\pm$0.40 &51.58$\pm$1.65 &47.66$\pm$0.97 &45.25$\pm$3.05   \\
  \hline
  \system &45.34$\pm$1.10 &52.05$\pm$0.56 &58.28$\pm$0.82 &64.84$\pm$1.57  \\
  \systemled  &\bf{47.91$\pm$1.06} & \bf{56.79$\pm$0.81} &\bf{59.88$\pm$1.07}& \bf{64.87$\pm$1.66}  \\
  \systemplus  &41.30$\pm$0.54 &54.34$\pm$1.44 &50.95$\pm$0.89 &62.51$\pm$1.68  \\
  \bottomrule
  \end{tabular}
 
 \label{tbl:acc_straggler}
\end{table}

\paragraph{Comments on Projection Layer}
As illustrated in MOON paper~\citep{li2021model}, the projection head adding before the classification layer is necessary for the method. We also conduct experiments by  adding the projection layer for MOON and \systemplus only, and found both of them performs even worse than FedAvg on CIFAR-100 due to the projection layer, as presented in Tab.~\ref{tbl:table_compare}. On CIFAR-10 and Tiny-Imagenet, MOON and \systemplus benefit from the projection head and significantly outperform other methods. But the \systemplus still outperforms MOON by 1.9\%, 1.5\%, and 2.1\% on CIFAR-10,CIFAR-100, and Tiny-ImageNet respectively when $\alpha$=0.1. And in Tab.~\ref{tbl:table1_complete}, we notice that with participation ratio growing, the performance of \system beats \systemplus, which also hints that adding MLP layer is not a general solution for all situations.
\paragraph{Effect of Network Architecture}
We conduct the experiments with different models. The improvement for ResNet shows that our proposed method could work on modern networks and the superiority compared to FedAvg becomes more clear as the scale of the model growing, which is supported by results conducted on Tiny-Imagenet (ResNet-50 is used). The advantage of FedProx is shown on the large model due to the parameters regularization. The performance of FedProx is 2.9\% higher than FedAvg and \system is over 2.7\% higher than FedProx when $\alpha$=0.1.  For NLP tasks, though FedAvg could easily achieve high performance based on the pre-trained model, our proposed model still shows its superiority when the fine-tuning task turns difficult, which is shown on SST-5 results.

\subsection{Ablation Studies}

\paragraph{Effect of Buffer Length Setting}
To understand the best buffer length setting for \system, we conduct the experiments on different buffer length as shown in Tab.~\ref{tbl:avg_teacher_buffer}. For \system, the communication cost would be constant regardless of the buffer size $M$, since the ensemble is proceeded on the server. Setting $M=5$ performs best on $\alpha=1$ and $\alpha=0.1$, which are higher than $M=1$ over 0.5\% on CIFAR-10. For \systemled, we find there is no more benefit by increasing the buffer length to 7, and it does not help significantly when the non-IID-ness is not a major issue or communication rounds are few. 

\begin{table}[!htb]
\centering
\small
\caption{Performance of different buffer length for \system, $C=0.2$.}
\begin{tabular}{lcccc}
\hline
  DataSet & Buffer Length &  $\alpha$=1 & $\alpha$=0.5 & $\alpha$=0.1\\
  \hline
   \multirow{4}{*}{CIFAR-10}& 1 & 78.58$\pm$0.61 & \bf{77.15$\pm$0.28} & 71.73$\pm$0.79  \\ 
    &3 & 78.63$\pm$0.19 & 76.91$\pm$0.05 & 72.12$\pm$0.41 \\
    &5 & \bf{79.14$\pm$0.72} & 76.54$\pm$0.17 & \bf{72.27$\pm$0.84}\\
    &7 & 78.72$\pm$0.73 & 77.02$\pm$0.43 & 72.24$\pm$0.21 \\
 \hline
  \multirow{4}{*}{CIFAR-100}& 1 & 42.77$\pm$0.37 & 41.76$\pm$0.06 & 36.50$\pm$0.12  \\ 
    &3 & 42.92$\pm$0.19 & 42.33$\pm$0.21 & 36.82$\pm$0.24 \\
    &5 & \bf{43.82$\pm$1.39} & 42.29$\pm$0.30 & \bf{36.83$\pm$0.20} \\
    &7 & 43.15$\pm$1.37 & \bf{42.44$\pm$0.32} & 36.65$\pm$0.64 \\
    \bottomrule
  \end{tabular}
 
 \label{tbl:avg_teacher_buffer}
\end{table}

We conduct the experiments on different buffer length for \systemled on CIFAR-10, C=0.2, and find the best buffer length varies under different settings as shown in Tab.~\ref{tbl:buffer_length}. 

\begin{table}[h]
\renewcommand\arraystretch{1.1}
\centering
\small
\caption{Performance of different buffer length for \systemled.}
\setlength{\tabcolsep}{2pt} 
\begin{tabular}{cccc}
\hline
  Buffer Length &  $\alpha$=1 & $\alpha$=0.5 & $\alpha$=0.1\\
  \hline
    1 & 78.58$\pm$0.61 & 77.15$\pm$0.28 & 71.73$\pm$0.79  \\ 
    3 & 78.48$\pm$0.39 & 76.63$\pm$0.40 & 72.27$\pm$0.05 \\
    5 & \bf{78.88$\pm$0.39} & \bf{77.44$\pm$0.66} & \bf{72.41$\pm$1.70}\\
    7 & 78.56$\pm$0.63 & 76.78$\pm$0.20 & 71.57$\pm$1.43 \\
    \bottomrule
  \end{tabular}
  
 \label{tbl:buffer_length}
\end{table}

\paragraph{Choice of Regularizer}
In this paper, the KL-divergence is used as the regularizer to control the discrepancy between the local logit and global logit distribution. We also consider the case when the regularization term being modified as means squared error (MSE) defined over the logit, and we tested the MSE reguarlizer on CIFAR-10, CIFAR-100 and SST-5 under $\alpha=0.1$ with the buffer size $M=1$. As shown in  Tab.~\ref{tbl:loss_type}, both choices outperform the FedAvg. However, when using the MSE as the regularizer, the accuracy is 1.6\% worse on CIFAR-10, comparable to SST-5, and  1.6\% higher on CIFAR-100 compare to the accuracy when using KL-divergence as the regularizer respectively. 
The convergence of the proposed algorithm is still guaranteed by the Lemma~\ref{lemma:wellposed} in the supplementary. 
\begin{table}[h]
\centering
\small
\caption{Performance of different loss types, under $\alpha=0.1$, $C=0.2$ setting, buffer size is set as $M=1$.}
\setlength{\tabcolsep}{3pt} 
\begin{tabular}{lccc}
\hline
  Loss Type & CIFAR-10 &CIFAR-100 & SST-5\\
  \hline
    None & 69.22$\pm$0.91 & 34.91$\pm$1.00 & 40.39$\pm$0.50 \\ 
    MSE  & 70.12$\pm$0.98 & \bf{38.19$\pm$0.77} & \bf{43.50$\pm$1.68} \\
    KL & \bf{71.73$\pm$0.79} & 36.50$\pm$0.12 & 43.24$\pm$1.80 \\
    \bottomrule
  \end{tabular}
 
 \label{tbl:loss_type}
\end{table}

\section{Conclusion}
In this paper, we propose a novel global knowledge distillation method \system that utilizes the knowledge learnt from past global models to mitigate the client drift issue. \system benefits from the ensemble and knowledge distillation mechanisms to produce a more accurate model.
We validate the efficacy of proposed methods on CV/NLP datasets under different non-IID settings. The experimental results show that \system\ outperforms several state-of-the-art methods and demonstrates the effect of reducing the drift issue.



\bibliography{aaai22}
\bibliographystyle{iclr2022_conference}

\newpage
\section{Proof of Convergence Analysis}
\begin{lemma}\label{lemma:wellposed}
The set $\{\w\in\R{d}| \w \text{ satisfies } \eqref{eq:inexactness}\}$ is non-empty.
\end{lemma}
\begin{proof}
Note that 
\begin{align*}
    m(\w;\w_{t})
    &=F_k(\w) + \frac{\gamma}{2n_k}\sum_{i=1}^{n_k} \KL(h_k(\w_{t},x_{ki}) || h_k(\w,x_{ki}))\\
    &\leq F_k(\w) + \frac{\gamma}{2n_k}\sum_{i=1}^{n_k} \frac{\norm{h_k(\w_t,x_{ki}) - h_k(\w,x_{ki}) }^2}{\min_{j\in\{1,\cdots, \mathcal{C}\}} h_k(\w,x_{ki})}\\
    &\leq F_k(\w) + \frac{\gamma L_h}{2\delta}\norm{\w_t - \w }^2\overset{\Delta}{=}\tilde m(\w;\w_{t}),
 \end{align*}
 where the first inequality follows from \cite[Lemma 11.6]{klemela2009smoothing} and the second inequality follows from the Assumption~\ref{ass.fun}.3. Notice that for any approximate solution $\w^k_{t+1}$ satisfies $\tilde m(\w^k_{t+1};\w_{t}) \leq \tilde m(\w^t;\w_{t})$, then
$$
m(\w^k_{t+1};\w_t) \leq \tilde m(\w^k_{t+1};\w_{t}) \leq \tilde m(\w^t;\w_{t}) = m(\w^t;\w^t),
$$
which implies \eqref{eq:inexactness} will be satisfied at least by the minimizer of $\tilde T_k(\w, \w^t)$.
\end{proof}

\textbf{Proof of Theorem~\ref{thm:convergence}.}
\begin{proof}
By Assumption\ref{ass.algo}, for the $k$th client, 
$$
\begin{aligned}
 & \nabla F_k(\w_{k}^{t+1}) + \frac{\gamma L_h}{\delta}(\w_{k}^{t+1}-\w^t) + e_{k}^{t+1} = 0\\
 & \norm{e_{k}^{t+1}} \leq \eta \norm{\w^{t}}.
\end{aligned}
$$
Provided $\gamma, \eta$ and $S$ are chosen as described, then one can use the same proof in \cite[Theorem 4]{li2018federated} to show that
$$
\E_{S_t}[f(\w_{t+1})] \leq f(\w_t) - \rho \norm{\nabla f(\w_t)}^2.
$$
Then take the total expectation with respect to all randomness and by telescoping, one reaches
$$
\rho \sum_{t=0}^{T-1}\E[\norm{\nabla f(\w_t)}]^2 \leq f(\w_0)-f(\w^*).
$$
Divide $T$ on both sides, then
$$
\min_{t\in[T]} \E[\norm{\nabla f(\w_t)}]\leq \frac{1}{T}\sum_{t=0}^{T-1}\E[\norm{\nabla f(\w_t)}]^2  \leq \frac{f(\w_0)-f(\w^*)}{\rho T}.
$$
\end{proof}

\section{Extended Experiments}


\paragraph{Toy Example for FedAvg Limitation Illustration}
Fig.~\ref{fig:toy_expamle} provides a illustration of the limitation in FedAvg. We consider a 4-class classification task with a 3-layer MLP, the optimizer used is Adam. The 2-dim data points are randomly generalized in (-4,4). We notice that for FedAvg, the model overfits the local data and results a bias for the global model, while adding the global models distillation term would relieve the overfitting phenomenon.

\begin{figure*}[htb]
    \centering
    \begin{subfigure}[b]{0.24\textwidth}
        \centering
        \includegraphics[width=\textwidth]{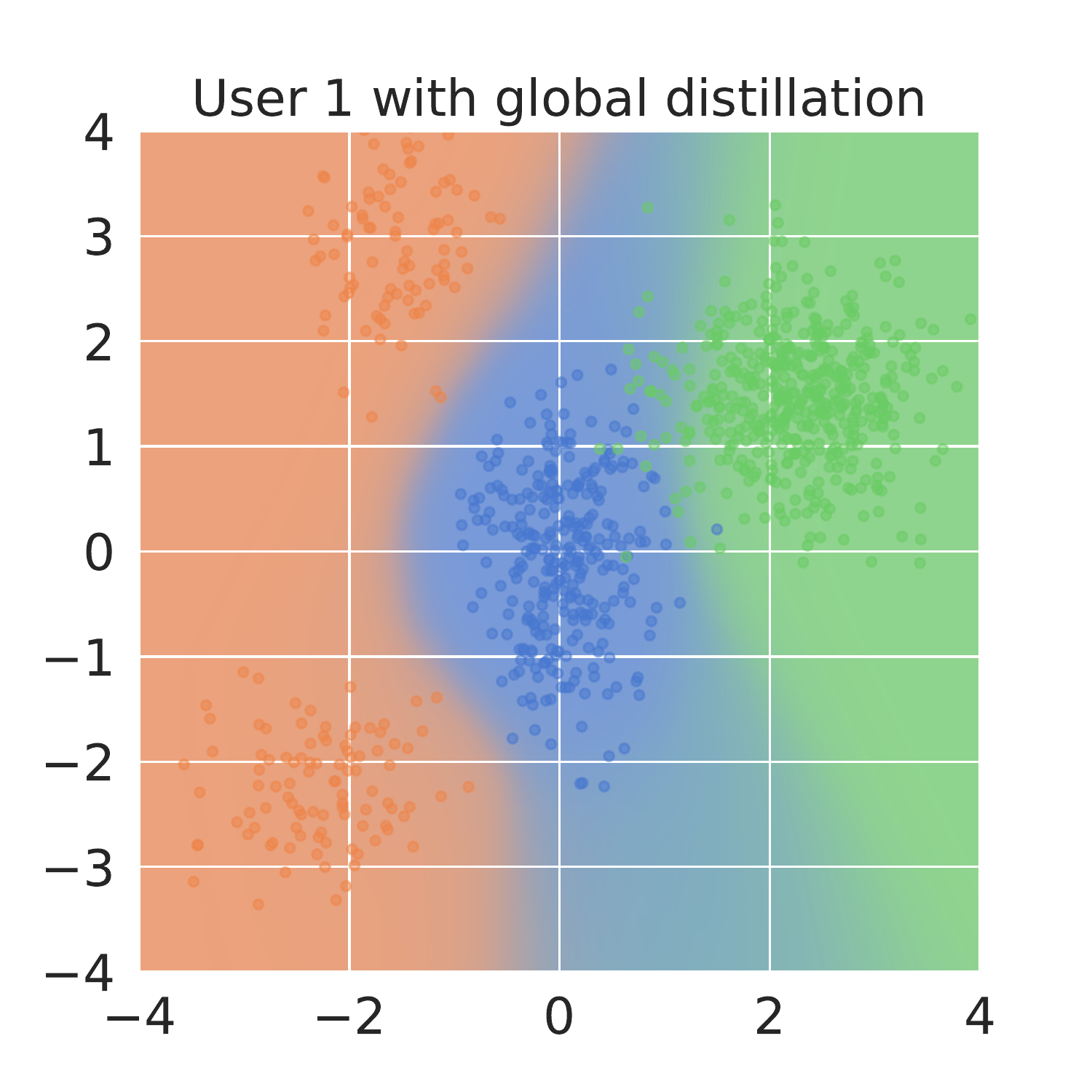}
        \caption{user 1}\label{fig:user_KD_bound_1}
    \end{subfigure}
    \begin{subfigure}[b]{0.24\textwidth}
        \centering
        \includegraphics[width=\textwidth]{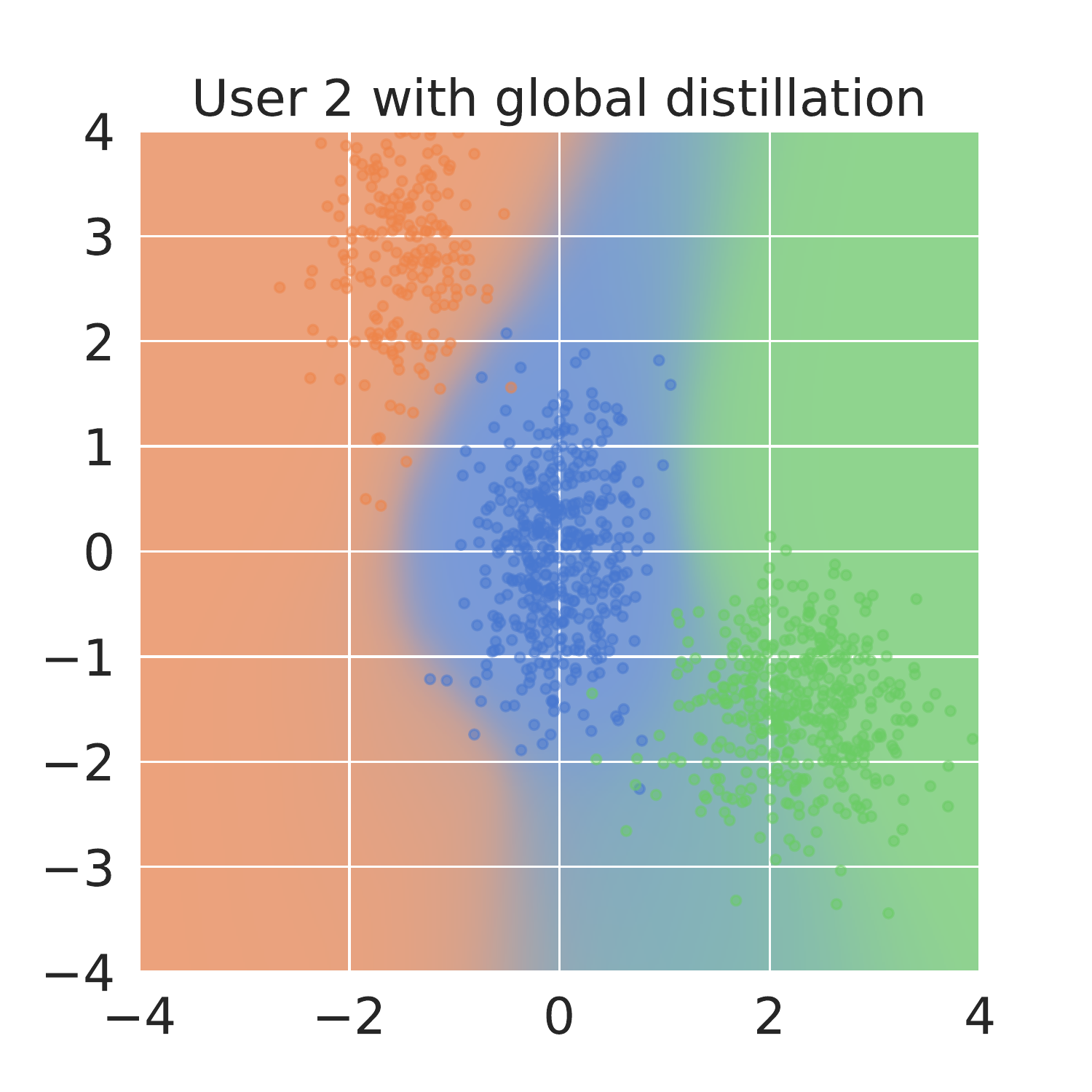}
        \caption{user 2}\label{fig:user_KD_bound_2}
    \end{subfigure}
    \begin{subfigure}[b]{0.24\textwidth}
        \centering
        \includegraphics[width=\textwidth]{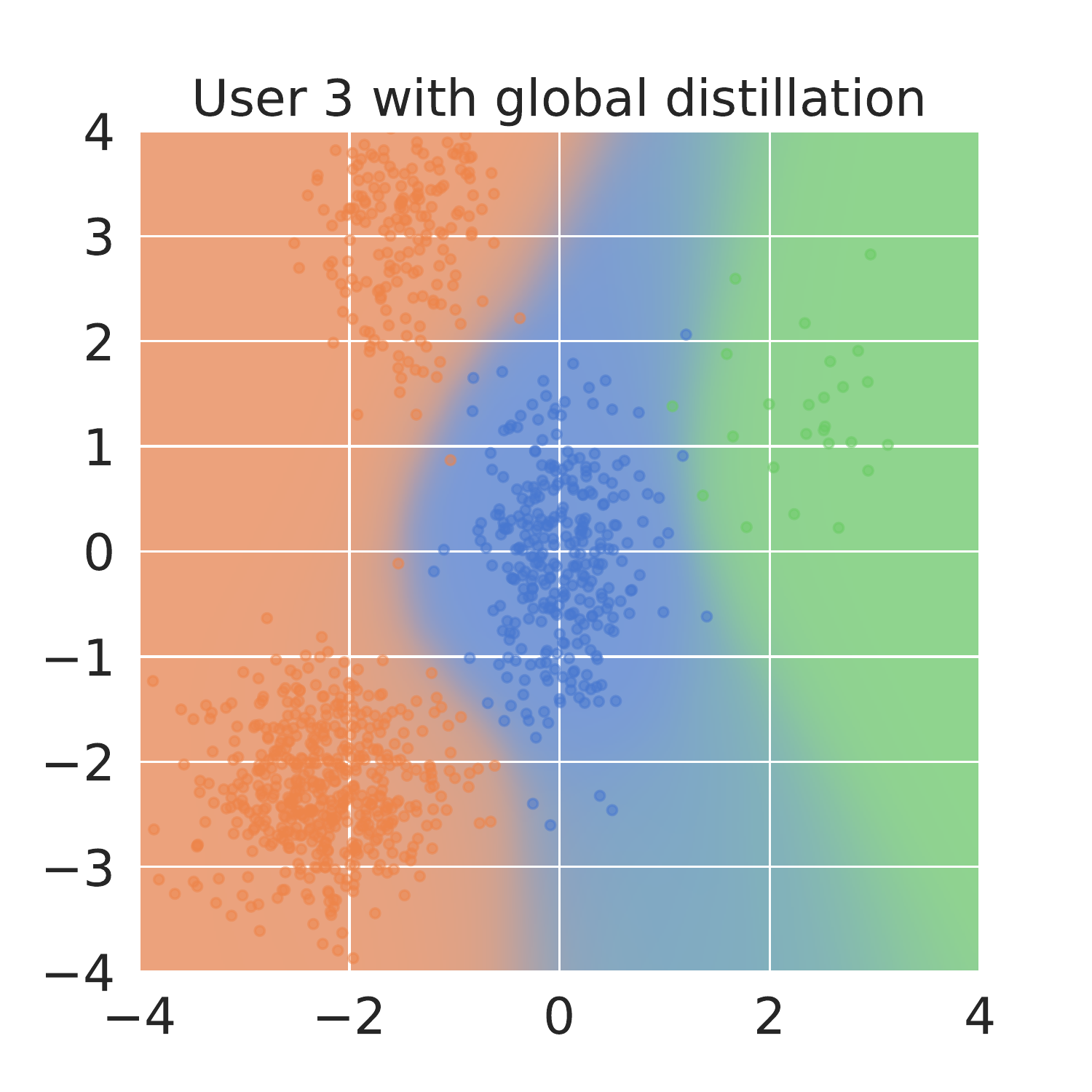}
        \caption{user 3}\label{fig:user_KD_bound_3}
    \end{subfigure} 
    \begin{subfigure}[b]{0.24\textwidth}
        \centering
        \includegraphics[width=\textwidth]{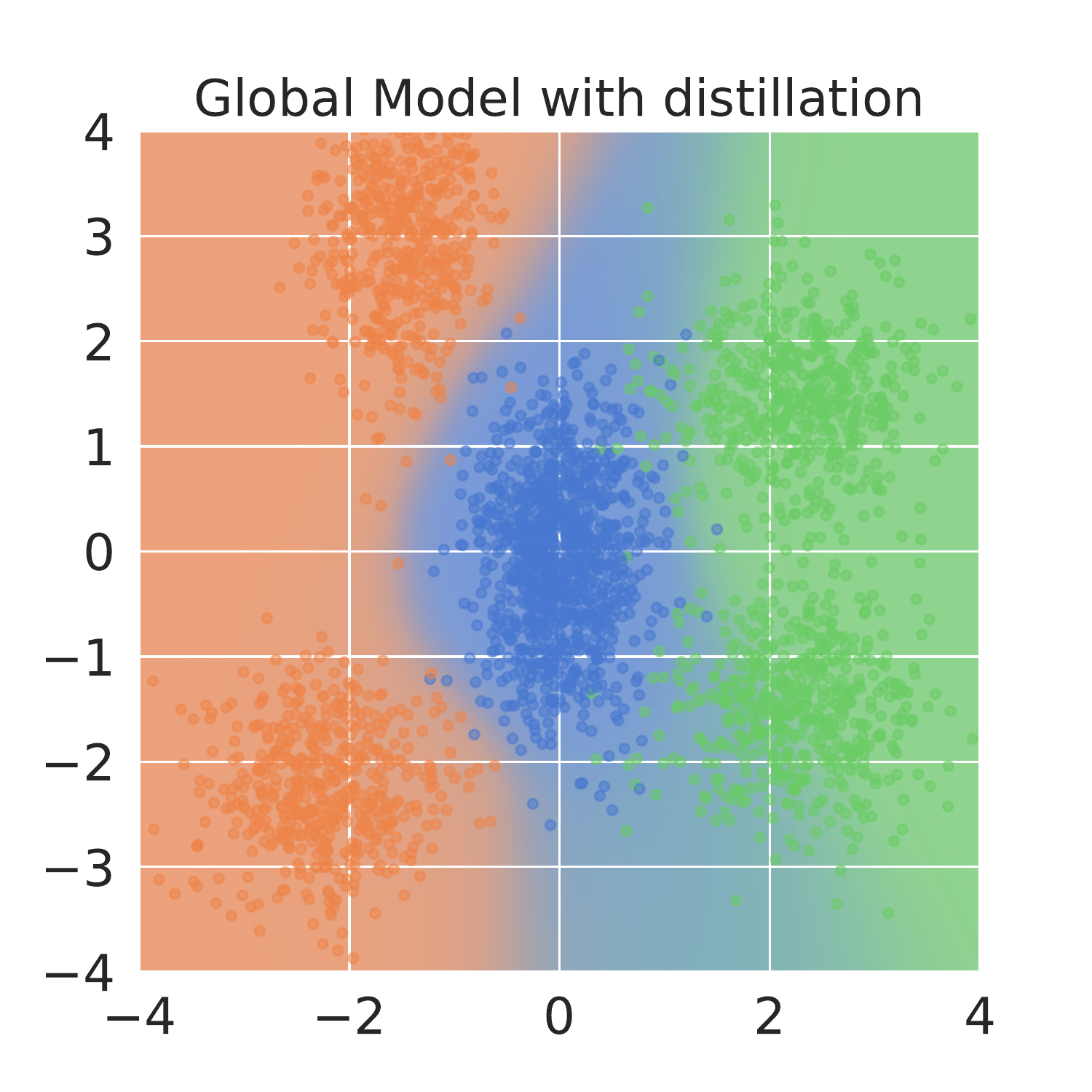}
        \caption{\system Bound}\label{fig:avg_KD_bound}
    \end{subfigure}
    \begin{subfigure}[b]{0.24\textwidth}
        \centering
        \includegraphics[width=\textwidth]{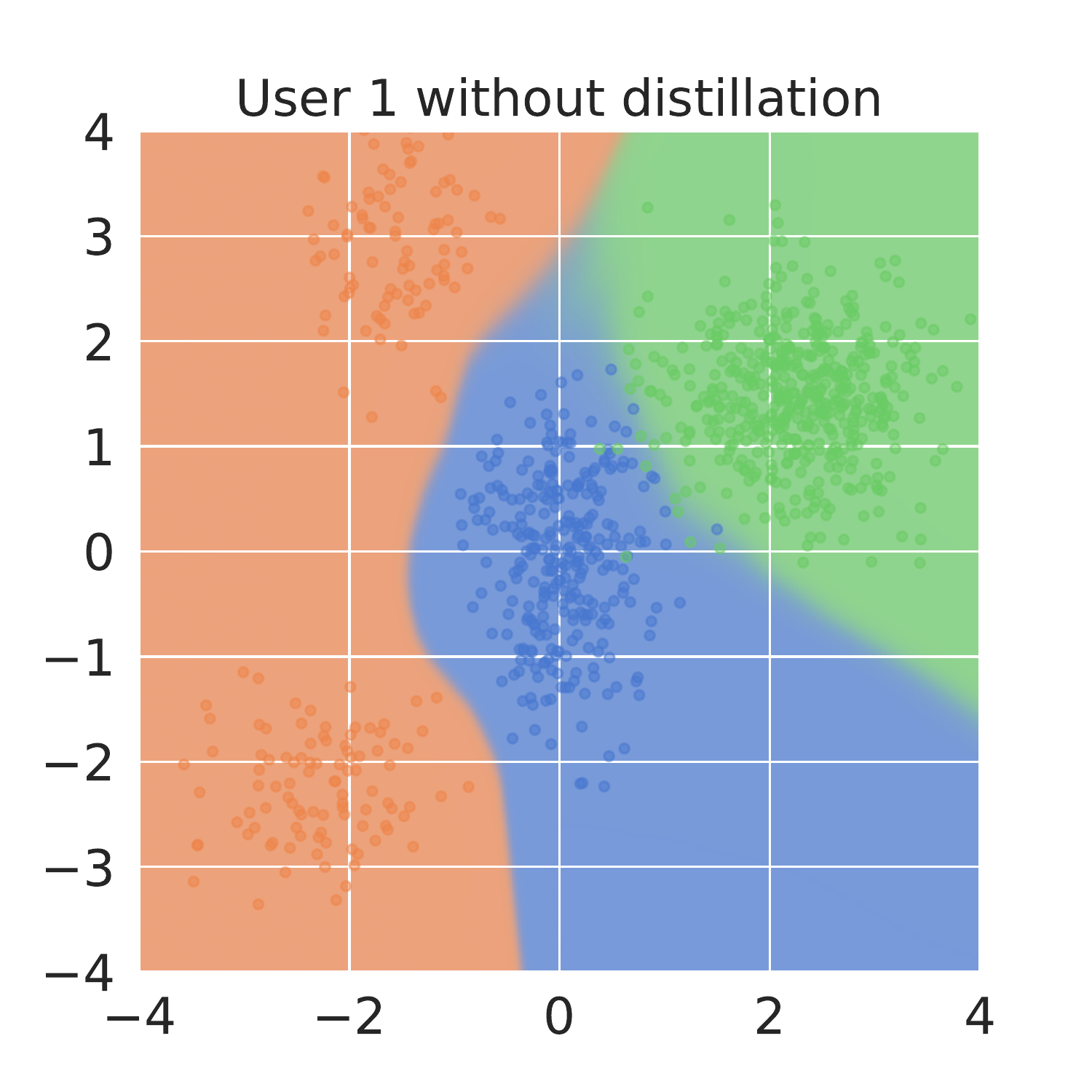}
        \caption{user 1}\label{fig:user_noKD_bound_1}
    \end{subfigure}
    \begin{subfigure}[b]{0.24\textwidth}
        \centering
        \includegraphics[width=\textwidth]{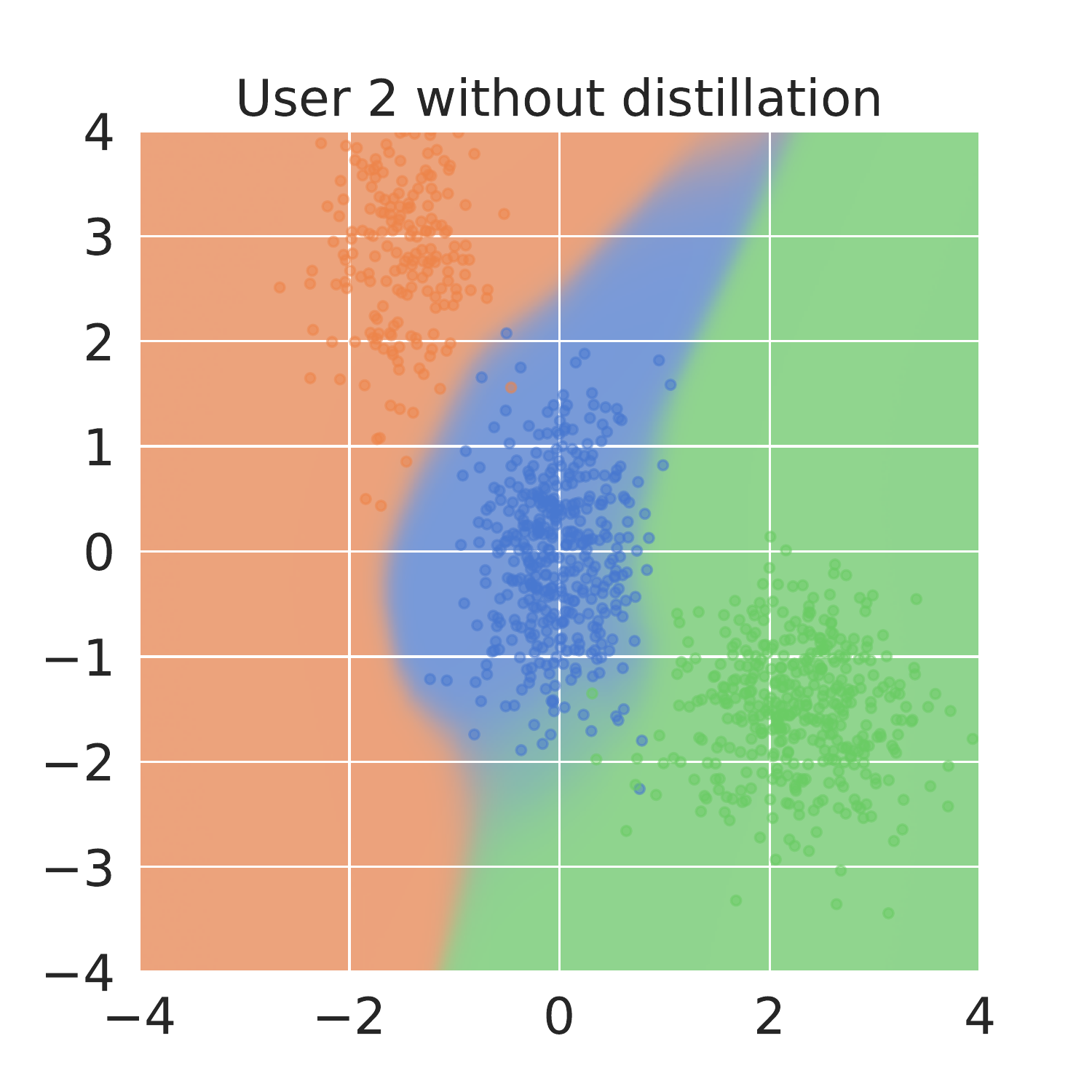}
        \caption{user 2}\label{fig:user_noKD_bound_2}
    \end{subfigure}
    \begin{subfigure}[b]{0.24\textwidth}
        \centering
        \includegraphics[width=\textwidth]{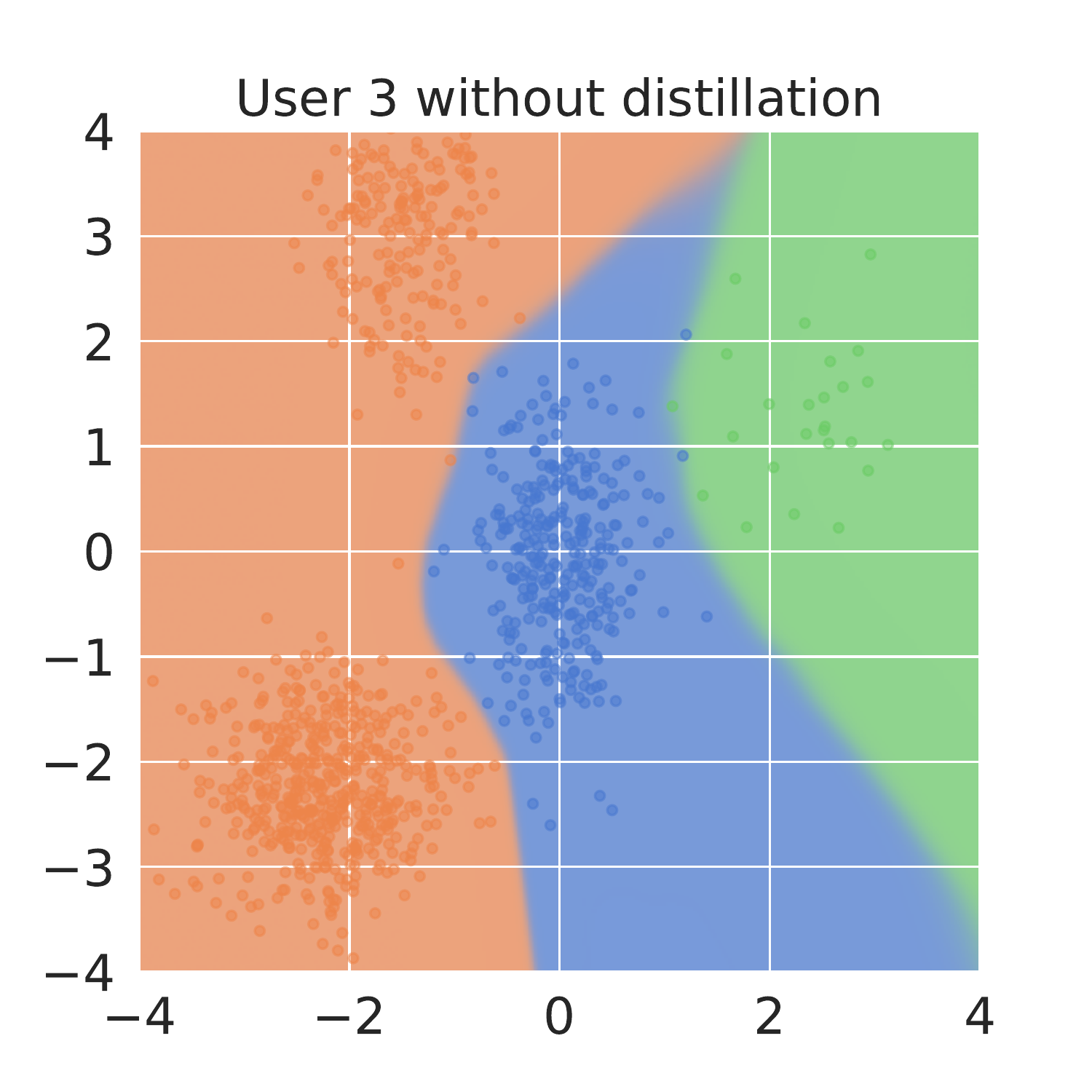}
        \caption{user 3}\label{fig:user_noKD_bound_3}
    \end{subfigure}
    \begin{subfigure}[b]{0.24\textwidth}
        \centering
        \includegraphics[width=\textwidth]{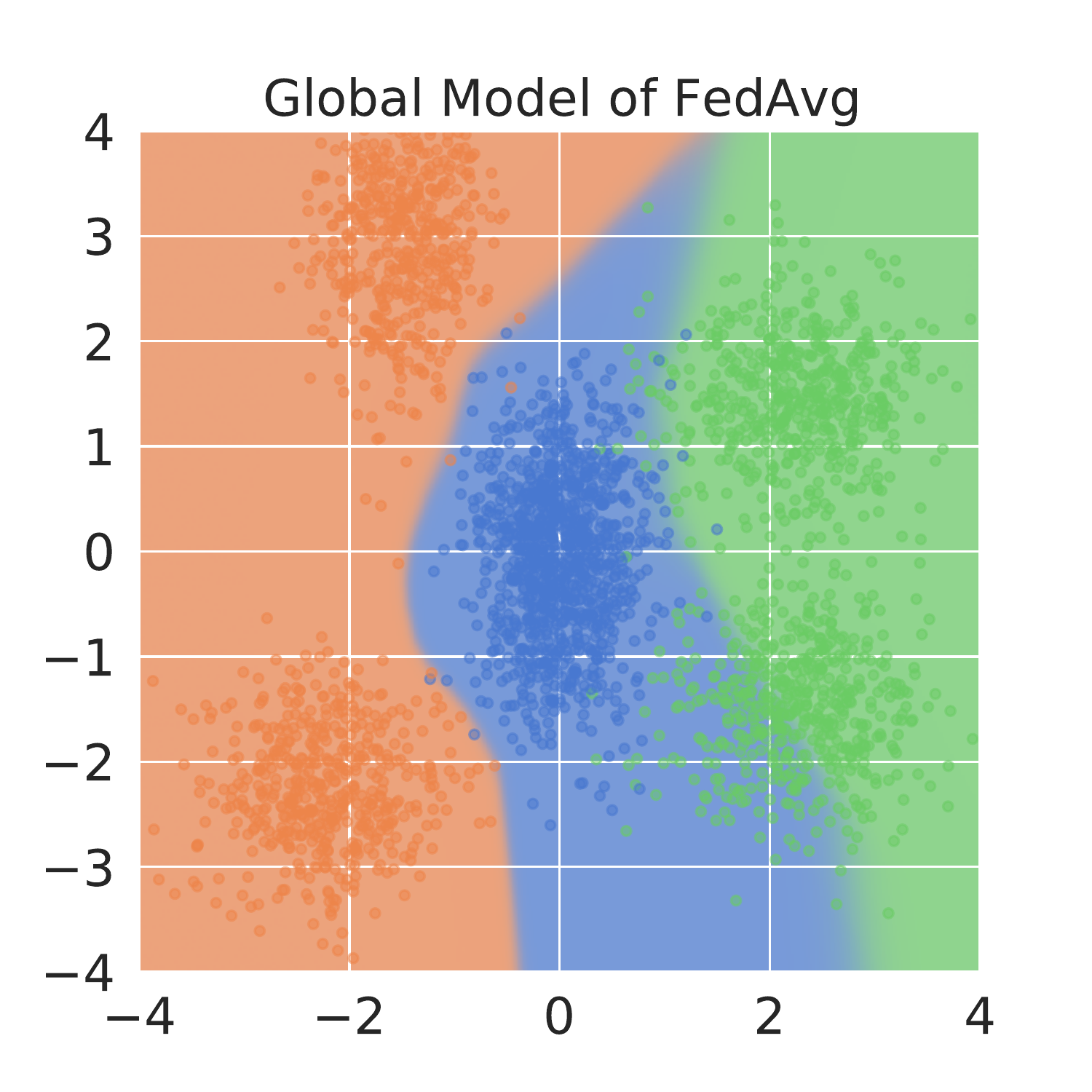}
        \caption{FedAvg Bound}\label{fig:avg_bound}
    \end{subfigure} 
    \caption{The decision bound for \system and FedAvg on toy examples.}\label{fig:toy_expamle}
\end{figure*}

\end{document}